\documentclass[]{siamart220329}

\pdfpagewidth=\paperwidth
\pdfpageheight=\paperheight

\usepackage[scr=boondox]{mathalfa}
\usepackage{amsmath,graphicx}
\usepackage{verbatim}
\usepackage{color}
\usepackage{amsfonts,amsbsy,amssymb}
\usepackage{natmove}
\usepackage{enumitem}
\usepackage{natbib}
\setcitestyle{numbers}

\def\br{\boldsymbol {\rm r}}

\def\d{\!\!{\rm d}}
\def\br{\boldsymbol {\rm r}}
\def\brho{{\boldsymbol {\rho}}}

\def\T{{\rm T}}
\def\A{{\boldsymbol{\rm A}}}

\def\B{\mathcal B}
\def\q{{{\mathscr{q}}}}
\newcounter{concount}
\newcounter{exampcount}
\newcounter{algcount}
\newtheorem{property}{Property}

\newsiamremark{remark}{Remark}
\newsiamremark{remarks}{Remarks}
\newsiamremark{keyremark}{{\bf Key Remark}}
\newsiamremark{note}{{\bf Note on Terminology}}
\newsiamremark{notation}{{\bf On Notation}}
\newsiamremark{example}{Example}
\newsiamthm{postulate}{Postulate}
\newsiamremark{construction}{Construction}
\DeclareMathOperator*{\argmin}{argmin}

\title{Analysis of Diagnostics (Part I):  Prevalence, Uncertainty Quantification, \& Machine Learning}

\author{Paul N.\ Patrone\footnotemark[1] 
\and Raquel A.\ Binder\footnotemark[2]
\and Catherine S.\ Forconi\footnotemark[2]
\and Ann M.\ Moormann\footnotemark[2]
\and Anthony J.\ Kearsley\footnotemark[1]
}

\begin{document}

\renewcommand{\thefootnote}{\fnsymbol{footnote}}

\footnotetext[1]{National Institute of Standards and Technology, Gaithersburg MD 20899, USA}
\footnotetext[2]{Department of Medicine, Division of Infectious Diseases and Immunology, University of Massachusetts Chan Medical School, Worcester,
MA, 01655, USA}

\maketitle

\begin{abstract}
Diagnostic testing provides a unique setting for studying and developing tools in classification theory.  In such contexts, the concept of \textit{prevalence}, i.e.\ the number of individuals with a given condition, is fundamental, both as an inherent quantity of interest and as a parameter that controls classification accuracy.  This manuscript is the first in a two-part series that studies deeper connections between classification theory and prevalence, showing how the latter establishes a more complete theory of uncertainty quantification (UQ) for certain types of machine learning (ML).  We motivate this analysis via a lemma demonstrating that general classifiers minimizing a prevalence-weighted error contain the same probabilistic information as Bayes-optimal classifiers, which depend on conditional probability densities.  This leads us to study \textit{relative probability level-sets} $B^\star(q)$, which are reinterpreted as both classification boundaries and useful tools for quantifying uncertainty in class labels.  To realize this in practice, we also propose a numerical, homotopy algorithm that estimates the $B^\star(q)$ by minimizing a prevalence-weighted empirical error.  The successes and shortcomings of this method motivate us to revisit properties of the level sets, and we deduce the corresponding classifiers obey a useful monotonicity property that stabilizes the numerics and points to important extensions to UQ of ML.  Throughout, we validate our methods in the context of synthetic data and a research-use-only SARS-CoV-2 enzyme-linked immunosorbent (ELISA) assay. 
\end{abstract}

\section{Appeal to the Reader}

Data analysis for diagnostic testing is not a task that one necessarily  associates with new applied mathematics.  Perusing the literature, one could be forgiven for believing that most of the hard problems have been reduced to identifying a ``cutoff value'' for classifying test results and estimating rates of false-positives and false-negatives.  Yet in the past few years, several of us have shown that many tasks in diagnostics are anchored in surprising ways at the intersection of measure theory, metrology, and optimization \cite{Patrone21_1,Patrone22_1}.  This perspective resolved several misconceptions and led to seemingly paradoxical (but true) results.  For example, classifiers need not, and at times should not, be used to estimate the number of individuals in a given class, i.e. the \textit{prevalence.} \cite{Patrone22_2}.

Motivated by such findings, this manuscript is the first in a series that further explores the interplay between probability and classification, using diagnostics to study their deeper relationship to uncertainty quantification (UQ).  When we first started on this path, our goal was more simply to minimize the assumptions needed to solve practical classification problems in diagnostic settings.  But this line of inquiry took us into unexpected territory, and we soon realized that the concept of prevalence, which is often ignored in classification and machine learning (ML), is in fact central to the whole enterprise.  Thus, in some sense, this work is really about the many interpretations of prevalence and how they connect certain topics of applied mathematics.  Given the range of issues that emerged, however, we felt that a single manuscript was not sufficient to cover them all.  The present manuscript therefore focuses on a fundamental connection between prevalence, optimization, UQ, and supervised ML techniques.  The second manuscript builds on this foundation to show how prevalence connects classification and unsupervised learning to aspects of linear algebra.  

Because many of these concepts arise naturally in epidemiology and diagnostics, we believe that such settings are well-suited for our exposition.  As a result, this manuscript is unabashedly interdisciplinary.  Typical readers should have a broad interest in the aforementioned fields of applied mathematics, with a willingness to ground the underlying ideas in the basics of \textit{serological assays (or blood-based diagnostic tests)}.  The latter contain elements common to many classification and ML problems.  We also feel that grounding the analysis in examples makes it easier to understand.  In this vein, we employ a hybrid writing style, staying within the confines of mathematics but adopting diagnostics-inspired terminology and notation.  

\section{Introduction and Motivation}
\label{sec:introduction}


Tasks in diagnostics share many features with more general classification problems and are therefore useful for illustrating key concepts.   \cite{Venables}.  Consider, for example, a setting wherein individuals $\omega$ in a population or sample space $\Omega$ can have some condition, e.g.\ an infection.  Instead of knowing the medical status of each individual, we are given a measurement outcome $\br(\omega)$, which can be interpreted as both a diagnostic test result and a random variable.   Epidemiologists seek to answer many questions on the basis of $\br(\omega)$, among them: ({\bf QI}) how many people have the condition; and ({\bf QII})  which ones have the condition?  The first question is referred to as a {\it prevalence estimation} problem;\footnote{Prevalence is the fraction of individuals with a condition.  In the context of serology testing, this is more appropriately called ``seroprevalence,'' which is the fraction of individuals with blood-borne markers of previous infection.  We use the words prevalence and seroprevalence interchangeably.} the second is clearly a classification problem.  See Refs.\ \cite{3Sig1,EUA,ROC,3Sig2,3Sig3}, for example.

In analyzing the general form of ({\bf QII}), many works have found it convenient to: (i) assume that $\br: \Omega \mapsto \Gamma$ for some set $\Gamma \subset \mathbb R^m$; and (ii) postulate the existence of a probability space induced by probability density functions (PDFs) of the $\br(\omega)$ conditioned on the true classes of the sample points \cite{Bayes,RW}.  Given this, the pointwise structure of the PDFs can be used to construct a partition $\Gamma$ whose elements define a minimum-error class assignment rule, e.g.\ via a Bayes optimal approach \cite{Bayes,RW}.  However, diagnostics points to two issues that are sometimes neglected in this more general setting.  First, the point-wise classification accuracy can be difficult to estimate because it strongly depends on the prevalence \cite{Patrone22_1}.  This is especially problematic, for example, when the prevalence is changing due to spread of a disease \cite{Bedekar22}.  Second, the conditional PDFs may be difficult to model, especially when $\br$ is a high-dimensional vector or there is limited training data.

The present manuscript addresses these problems by identifying a duality between prevalence and the conditional PDFs that simultaneously simplifies the modeling process and realizes point-wise uncertainty estimates.  We begin by studying the properties of a prevalence-weighted classification error, which we use as an objective function for training classifiers.  This leads to a lemma that motivates an isomorphism between certain low-dimensional boundary sets $B^\star$ and classifiers defined via partitions of $\Gamma$.  We then formulate a homotopy method that approximates such $B^\star$ in terms of empirical data, thereby generalizing techniques such as one-dimensional (1D) receiver-operating characteristics \cite{ROC} and support vector machines (SVM) \cite{SVM1}.  The successes and shortcomings of this approach also motivate a deeper analysis of its connection to probability.  Here we arrive at our three main results: (i) the boundary sets are in fact \textit{relative conditional probability level sets}; (ii) for fixed $\br$, the optimally assigned class $\hat C(\br)$ is a monotone function of prevalence; and thus (iii) the relative conditional probability and uncertainty in the true class can be determined from the value of prevalence at which $\hat C(\br)$ changes.  We also discuss how this \textit{level-set theory} modifies our numerical methods and unifies certain types of ML algorithms.  Throughout, these ideas are validated in the context of both synthetic data and the SARS-CoV-2 enzyme-linked immunosorbent assay (ELISA) developed in Ref.\ \cite{Raquel1}.\footnote{The assay is for serological research-use-only.}

A key challenge in this work is to distinguish the interpretation of prevalence as a property of a population from its role as a parameter used in optimization.  We only consider settings for which it is meaningful to assume that a test sample has some \textit{a priori} probability $q$ of belonging to a given class.  But by the same token, we may always ask, ``what are the prevalence values for which $\br$ is optimally assigned a chosen class?''  The latter perspective does not consider a specific population per se, but rather a family of them parameterized by $q$.  We ultimately connect these interpretations through the law of total probability \cite{totprob}, and our main results reveal a third way of understanding prevalence in terms of classification accuracy.  

The practical motivation for our line of inquiry arises from the way in which empirical distributions overcome issues of quantifying and controlling {\it model-form} errors; see Refs.\ \cite{MLUQ1,BoisvertUQ,SmithUQ} for a general overview of the latter.  In typical assay design settings, one may only have access to $\mathcal O(100)$ or fewer training samples with which to construct the conditional PDFs.  This limits the usefulness of spectral methods, for example, which are unbiased but  have a slow, mean-squared rate of convergence \cite{SMC2,Patrone17,SmithUQ,SMC3,SMC1}.  More tractable approaches rely on techniques such as maximum likelihood estimation for parameterized probability distributions \cite{MLE,Luke23_1,RW,SmithUQ}, but in practice, their uncertainty can increase dramatically with dimension (e.g.\ number of antigens).  The empirical distribution is thus a desirable object with which to work because its samples are generated from the underlying true (but unknown) distribution, which reduces the need for subjective modeling.  Ultimately, our hope is that the resulting uncertainties are tied directly to the amount of sampling and therefore remain controllable.

This hope, however, also points to the key limitations of our numerical methods.  One must make some choices to avoid trivially over-fitting the classification boundaries.  In the examples herein, we assume that these boundaries are well described by low-order polynomials, since this facilitates optimization, works well with many biological datasets, and avoids introducing high-frequency structure.  This choice is also more general and flexible than postulating parameterized forms of the conditional PDFs.  However, our approach still assumes structure of the underlying data, and it is not clear when and to what extent the analysis is unbiased and/or converges.  We numerically address these questions in the context of specific examples, although a deeper treatment remains an open problem.

Our analyis is also restricted to a binary setting.  We make this choice for two reasons.  First, the binary classification problem is sufficiently rich that we feel that it warrants a study of its own.  Second, we anticipate that multiclass problems can be addressed by extension of the ideas developed herein.  Thus, while we point to generalizations of our analysis to multiclass settings, we leave an in-depth study of such issues for future work.

The rest of the manuscript is organized as follows.  Section \ref{sec:setting} establishes the mathematical setting of diagnostic classification and formally relates boundary sets to classifiers defined via partitions.  Section \ref{sec:classification} formulates our classification algorithm from a general perspective and in the context of a quadratic model.  Section \ref{sec:validation} validates aspects of the analysis for synthetic and real-world data.  Section \ref{sec:UQ} establishes the level-set interpretation of our classifier and formulates constraints that yield uncertainty estimates. Section \ref{sec:discussion} makes deeper connections with previous works and discusses limitations and open directions of our analysis.

\section{Mathematical Setting}
\label{sec:setting}

\subsection{Notation}
\label{subsec:notation}

We leverage the following conventions from diagnostics.
\begin{itemize}
\item[(a)] except when referring to the sign of a number, the terms ``negative'' and ``positive'' generally denote some condition, e.g.\ an individual having a certain type of antibody in his or her blood sample indicating past infection;
\item[(b)] when used as subscripts, the letters $n$ and $p$, which we often substitute for $0$ and $1$, denote ``negative'' and ``positive'' in the sense of (a);
\item[(c)] likewise, the functions $N(\br)$ and $P(\br)$ are (conditional) probability densities associated with negative and positive populations in the sense of (a).
\end{itemize}

We also employ the following notation throughout.
\begin{itemize}
\item Bold lowercase Roman and Greek letters (e.g.\ $\br$) denote column vectors; non-bold versions are scalars.
\item Subscripts $i$, $j$, and $k$ attached to vectors denote random realizations of a vector, not its components.
\item The symbols $D$ and $B$  always refer to sets.  $B$ always denotes a ``boundary set.''  Distinguish this from $\mathcal B(\br)$, which denotes a representation of $B$ in terms of a function.  We always equate $B=\{\br:\mathcal B(\br)=0\}$.
\item Bold uppercase letters (e.g.\ $\A$) denote matrices; subscripted, non-bold versions are matrix elements.  For example, $A_{i,j}$ is the $(i,j)$th element of $\A$.
\item The notation $X_D$ means
\begin{align}
X_D = \int_{D}\d\br \,\,\, X(\br)
\end{align}
where $X$ is always a probability density function and $D$ is a set.  That is, $X_D$ is the measure of set $D$ with respect to the PDF $X(\br)$.  
\item The acronym {\it iid} always means ``independent and identically distributed.''
\end{itemize}

\subsection{Key Assumptions}
We always assume absolute continuity of measure \cite{Tao}.  In particular, we always assume that individual points $\br\in \Omega$ have zero measure with respect to all distributions.

\subsection{Background Theory}
\label{subsec:bg}

Consider a diagnostic assay that is used to determine the properties of a {\it test population} or sample space $\Omega$.  Let individuals $\omega\in \Omega$ of this population belong to one of two classes $C(\omega)$, referred to colloquially as ``negative'' and ``positive.''  Note that $C(\omega)$ is a discrete random variable.  In practice, the class $C(\omega)$ of an individual in a test population is unknown.  Instead, we are given a different random variable $\br(\omega)$, which we interpret as a measurement (i.e.\ a diagnostic result) $\br$ in some set $\Gamma \subset \mathbb R^n$.  Given only $\br(\omega)$, the goal of classification is to deduce $C(\omega)$ of each individual, ideally with the highest possible accuracy.  The goal of prevalence estimation is to determine the fraction of positive individuals in the population.  In both cases, we refer to the collection of measurements being analyzed as the {\it test data}.    In the following, we work directly with the random variables $\br$, typically omitting explicit dependence on $\omega$ or the underlying probability space, i.e.\ as defined in terms of $\Omega$, a $\sigma$-algebra, and a measure.  See Refs.\ \cite{Evans,pspace} for more rigorous context on such issues.

\begin{example}[Antibody Assays]
Antibodies are immune molecules created by the body in order to disable or identify pathogens such as viruses.  In order to do so, antibodies typically undergo a process of hyper-evolution that enables them to bind selectively to viral proteins or structures that mediate infection of a cell, for example \cite{hypermutation}.  Moreover, antibodies in blood are often long-lasting.  For this reason, they are frequently the target of blood-based or serological measurements that seek to determine if a person has previously been infected by some pathogen \cite{genantibody}.  

\begin{figure}
\includegraphics[width=8cm]{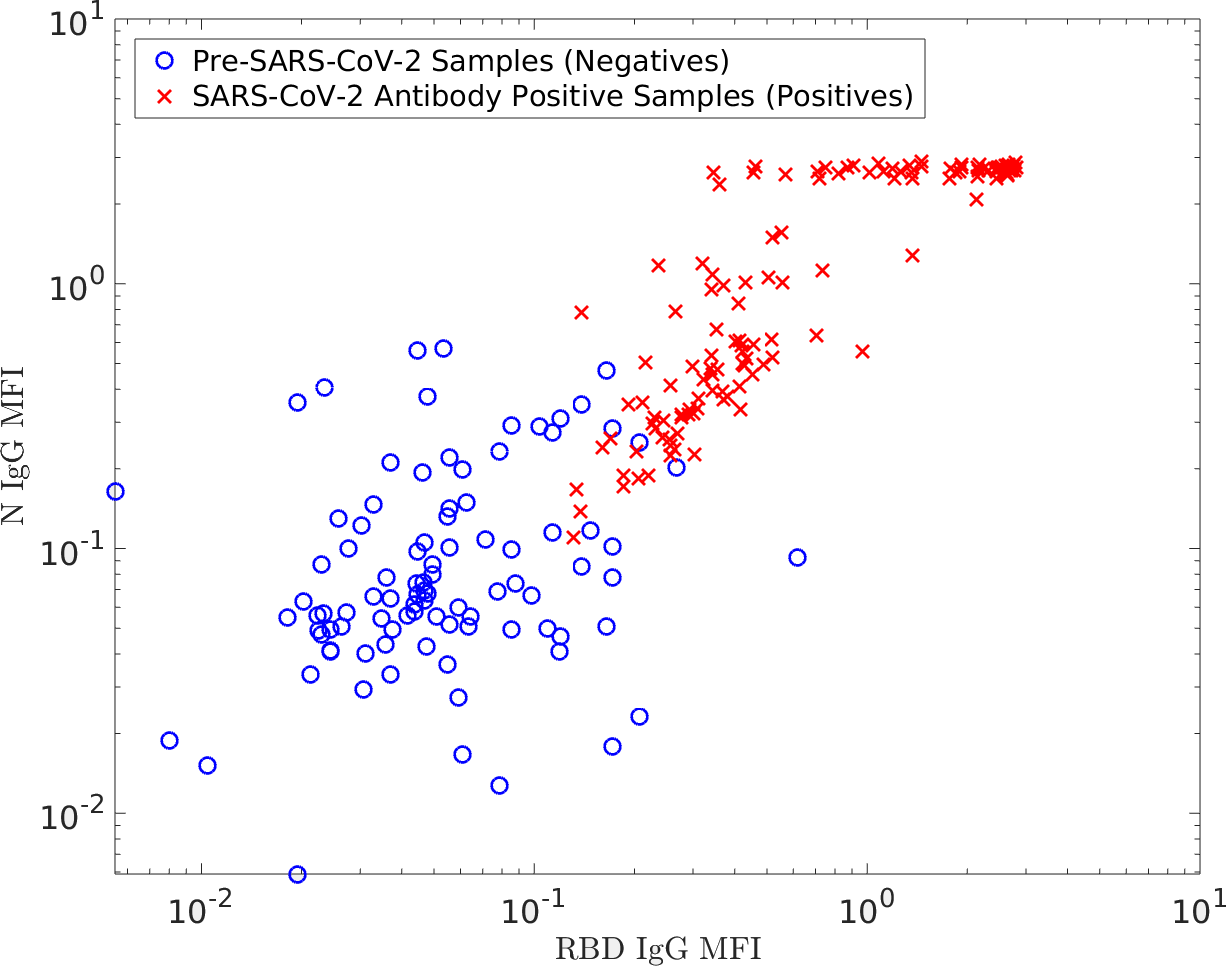}\caption{Representative output of an antibody assay.  Each point corresponds to a random variable $\br(\omega)$ whose underlying sample point $\omega$ is an individual that donated blood.    The $\br(\omega)$ are measurement outcomes of a diagnostic test that quantifies the amount of two types of antibodies that bind to different parts of the SARS-CoV-2 virus.  The horizontal axis is the scale for the dimensionless mean fluorescence intensity (MFI) measurement that quantifies the amount of receptor-binding-domain (RBD) immunoglobulin g (IgG) antibodies in each sample.  The vertical axis sets the corresponding scale for nucleocapsid (N) IgG antibodies.  In a typical diagnostic setting, the underlying true class $C(\omega)$ of the individual is unknown.  In the figure, however, extra information allows us to classify the data as either negative (blue \textcolor{blue}{o}) or positive (red \textcolor{red}{x}).  Thus, this dataset is an example of {\it pure training data}; cf.\ Definition \ref{def:pure}.  See Ref.\ \cite{Raquel1} for additional details of this data.}\label{fig:schematic}
\end{figure}

Figure \ref{fig:schematic} illustrates a typical collection of serology measurements, where each point $\br(\omega)$ corresponds to a sample taken from a different individual $\omega$, and each dimension of $\br(\omega)$ is a different antibody type.  Diagnosticians often seek to answer Questions Q.I and Q.II based on this information alone.  Note that in the figure, the true classes are given, whereas for a test population, $C(\omega)$ is unknown.
\end{example}

\begin{remark}
Figure \ref{fig:schematic} contains many points for which $\br(\omega) \approx (30,30)$.  This value is likely an upper limit of a photodetector used in the measurement process.  Thus, if antibody concentrations are unbounded, this dataset appears to violate our assumption that all $\br(\omega)$ have zero measure.  However, note that the instrument {\it noise} has a wonderfully regularizing effect on the data, in effect smearing out the singularity. 
\end{remark}

To construct classifiers and estimate prevalence, it is useful to characterize the distribution that generates the test data.  Several definitions are in order.
\begin{definition}[Prevalence Convention]
Let $\omega \in \Omega$ be a sample point and $C(\omega)\in \{0,1\}$ be a binary random variable.  We adopt the convention that $C(\omega)=0$ is a ``negative'' sample and $C(\omega)=1$ is a ``positive'' sample.  Moreover, we define the {\bf prevalence} to be the probability $q$ that $C(\omega)=1$.
\end{definition}

\begin{definition}[PDF Conventions]
Let $\omega \in \Omega$ be a sample point and $\br(\omega)\in \Gamma \subset \mathbb R^m$ be a random variable for some positive integer $m$.  If $C(\omega)$ is a binary random variable, we refer to $P(\br)$ as the {\bf conditional probability density function} of $\br(\omega)$ conditioned on $C(\omega)=1$.  That is, it is the probability that a positive sample yields measurement value $\br$.  Likewise, $N(\br)$ is the conditional PDF for a negative sample.  
\end{definition}
\begin{definition}[PDF of a Binary Test Population]
Assume that $C(\omega)$ is a binary random variable, and let $q \in [0,1]$.  We define 
\begin{align}
Q(\br;q)=qP(\br)+(1-q)N(\br)  \label{eq:qdef}
\end{align}
to be the PDF that a {\bf test sample} (whose class is unknown) yields measurement $\br$.  We refer to the underlying sample space from which the $\br$ are generated to be a {\bf test population} with prevalence $q$. 
\end{definition}
\begin{remark}
Equation \eqref{eq:qdef} is a realization of the law of total probability \cite{totprob}.
\end{remark}

Reference \cite{Patrone21_1} showed that prevalence estimation follows directly from the conditional PDFs and Eq.\ \eqref{eq:qdef} via the following construction.

\begin{lemma}[Class-Agnostic Prevalence Estimator]
Let $C(\omega)$ be a binary random variable and assume $P(\br)$ and $N(\br)$ are known.  Let $D \subset \Gamma$ be a subdomain chosen such that the difference of measures $P_D$ and $N_D$ is non-zero; i.e.\ $|P_D-N_D|>0$.  Let there be $s$ {\it iid} random variables $\br_i$ {\rm (}$i \in \{1,2,...,s\}${\rm )} drawn from $Q(\br;q)$.  Then 
\begin{align}
\tilde q = \frac{\tilde Q_D - N_D}{P_D - N_D}, & \qquad \quad \qquad{\rm where }& \tilde Q_D = \frac{1}{s}\sum_{i=1}^s \mathbb I(\br_i \in D) \label{eq:prevest}
\end{align}
is an unbiased estimate of $q$ that converges in mean-square as $s\to \infty$, where $\mathbb I$ is the indicator function. \label{lem:qprev}
\end{lemma}  

Lemma \ref{lem:qprev} follows by integrating Eq.\ \eqref{eq:qdef} and approximating $Q_D$ in terms of $\tilde Q_D$.  The latter is a binomial random variable, which implies that $\tilde q$ is unbiased and converges.  See Ref.\ \cite{Patrone22_2} for more details, as well as a method for optimizing $D$.

\begin{remark}
The domain $D$ in Lemma \ref{lem:qprev} has nothing to do with classification.  The only requirement is that $|P_D-N_D|>0$, which simply states that one probability distribution must have more mass in $D$ than the other.  The (deterministic) term $N_D/(P_D-N_D)$ is a correction factor that removes bias in the naive prevalence estimator $\tilde Q_D$, which is often used when $D$ corresponds to a classification domain (see Def.\ \ref{def:classifier}).  Lemma \ref{lem:qprev} and Refs.\ \cite{Luke23_1,Patrone22_2} generalize the results of Ref.\ \cite{OldPrevOpt}.  
\end{remark}

To classify a test sample, it is common to identify two domains $D_p$ and $D_n$, where $\br\in D_p$ or $\br\in D_n$ is interpreted as $\br$ corresponding to a positive or negative individual (although this is a choice, not an objective truth).  We make this precise as follows.
\begin{definition}
Let $U=\{D_j\}$, $j\in\{0,...,K-1\}$ be a partition of $\Gamma$ with $K$ elements, meaning that
\begin{subequations}
\begin{align}
\bigcup_{j=0}^{K-1} D_j &= \Gamma, \\
 D_j \cap D_k &= \emptyset & j \ne k.
\end{align} 
\end{subequations}
We refer to $U$ as a {\bf classifier} and assign $\br$ to class $j$ if $\br \in D_j$.  That is, $U$ induces a new random variable $\hat C(\br(\omega),U)$ whose value is $j$ when $\br \in D_j$.  When the number of classes $K=2$, we refer to $U$ as a {\bf binary classifier}.
\label{def:classifier}
\end{definition}

\begin{remark}[Equivalence Class of Partitions]
Definition \ref{def:classifier} ensures that a measurement $\br$ is always assigned a unique class.  However, this requirement is slightly stronger than what is needed in practice.  Because each $\br$ is assumed to have zero measure, we may move any countable set of them between classes, or even remove them from $U$ itself, without meaningfully affecting data analysis; that is, only manipulations on sets of positive measure are important.  In the context of a binary classifier, we therefore replace Def.\ \ref{def:classifier} with the requirements that
\begin{align}
\int_{D_0 \cap D_1} \hspace{-7mm}Q(\br;q)\,\,\, \d \br = 0, && \int_{D_0 \cup D_1} \hspace{-7mm}Q(\br;q)\,\,\, \d \br = 1, 
\end{align}
which allows us to work with equivalence classes of partitions that only differ on sets of measure zero.  Here we always adopt the more general perspective that $U$ is only defined up to such an equivalence class.
\end{remark}

An important observation of Ref.\ \cite{Patrone21_1} is that the average error rate of a diagnostic assay is a scalar-valued function of a partition.  In  the binary case:
\begin{lemma}[Binary Classification Error]
Let $U=\{D_n,D_p\}$ be a binary classifier.  The mapping ${\mathcal E: U \to [0,1]}$ given by
\begin{align}
\mathcal E(U;q) = (1-q)N_{D_p} + qP_{D_n} \label{eq:error}
\end{align}
is the average or {\bf expected classification error} associated with $U$ for a test population having prevalence $q$.  \label{def:classerror}
\end{lemma}
The proof of Lemma \ref{def:classerror} follows by decomposing $\Pr[\hat C(\br;U) \ne C(\omega)]$ via
\begin{align}
\Pr[\hat C(\br;U) \ne C(\omega)] &= \Pr[\hat C(\br;U) = 1 \cap  C(\omega)=0] + \Pr[\hat C(\br;U) = 0 \cap  C(\omega)=1] \nonumber \\
&=\Pr[\hat C(\br;U) = 1 | C(\omega)=0]\Pr[C(\omega)=0] \nonumber \\ & \qquad + \Pr[\hat C(\br;U) = 0 |  C(\omega)=1]\Pr[C(\omega)=1]. \nonumber
\end{align}
The result then follows by the definition of prevalence.

\begin{note}
In diagnostics,  $N_{D_p}$ and $P_{D_n}$ are often called the rates of false-positives and and false-negatives.  
\end{note}

Like Eq.\ \eqref{eq:qdef}, the average classification error is a realization of the law of total probability.  Equation \eqref{eq:error} also depends on $q$, so that the natural order of tasks is to estimate prevalence (e.g.\ via Lemma \ref{lem:qprev}) before considering classification.  Given this, the following well-known result yields the partition $U$ that minimizes $\mathcal E$  \cite{Patrone21_1,Patrone22_2,RW}.
\begin{lemma}[Optimal Binary Classifier]
Assume  $q\in [0,1]$.  Then the sets 
\begin{subequations}
\begin{align}
D_p^\star(q) &= \{r:qP(\br) > (1-q)N(\br) \} \cup  B_p^\star(q) \label{eq:Dp}\\
D_n^\star(q) &= \{r:qP(\br) < (1-q)N(\br) \} \cup  B_n^\star(q) \label{eq:Dn}
\end{align}
\end{subequations}
minimize $\mathcal E(U;q)$, where $B_p^\star(q)$ and $B_n^\star(q)$ are an arbitrary partition of the set $B^\star(q)= \{r:qP(\br) = (1-q)N(\br) \}$.  That is, Eqs.\ \eqref{eq:Dp} and \eqref{eq:Dn} define an equivalence class of partitions minimizing $\mathcal E(U;q)$. \label{lem:optclass}
\end{lemma}

\begin{keyremark}
Observe that the prevalence $q$ is an objective property of a test population.  In contrast, $\mathcal E(U;q)$ depends on objective and subjective quantities (i.e.\  $U$).  However, we could just as well let $q$ take any value on $[0,1]$ in Eqs.\ \eqref{eq:error}--\eqref{eq:Dn}, which would alter our notion of accuracy.  Thus $q$ appearing in $\mathcal E(U;q)$, $D_p^\star(q)$, and $D_n^\star(q)$ should be viewed as a choice that reflects a desire to minimize the average classification error.  While this objective function may seem arbitrary, its structure plays a fundamental role in later sections on UQ.  {\it We advise the reader to remain aware of the tension between the distinct interpretations of $q$ as a property of a population and a  parameter of $\mathcal E(U;q)$.} \label{rem:prev_interp}
\end{keyremark}
  
Given a binary classifier $U$, Ref.\  \cite{Patrone22_1} introduced the following UQ concept.
\begin{definition}
Let $U=\{D_p,D_n\}$ be a binary classifier.  Then for a test population with prevalence $q$, we define the {\bf local accuracy} of $U$ to be
\begin{align}
Z(\br;q,U) = \frac{qP(\br)\mathbb I(\br\!\in\! D_p) + (1-q)N(\br)\mathbb I(\br \!\in\! D_n)}{Q(\br;q)}. \label{eq:localacc}
\end{align}
If $U=\{D_p^\star,D_n^\star\}$, then we use the symbol $Z^\star(\br;q)$. 
\label{def:locacc}
\end{definition}

\noindent {\bf Interpretation:} $Z(\br;q,U)$ is the probability of correctly classifying a sample having measurement $\br$.   In diagnostics, $Z(\br;q,U)$ is sometimes called the post-test probability of having a condition.  However, this concept is not uniformly defined, nor is it always linked to the concept of conditional PDFs; see Refs.\ \cite{posttest1,posttest2,posttest3,posttest4} for conflicting definitions.  Note that $\mathcal E(U;q)=1-\boldsymbol {\rm E}[Z(\br;q,U)]$, where the expectation $\boldsymbol {\rm E}$ is with respect to $Q$. This justifies referring to $Z(\br;q,U)$ as the ``local accuracy.''

\subsection{A Motivating Lemma and Two Postulates}
\label{subsec:motivation}

Equations \eqref{eq:prevest} and \eqref{eq:error} imply that the conditional PDFs are central to interpreting test data, which begs the question: how does one construct $P(\br)$ and $N(\br)$?  In practice, this requires {\it training data}, which is often assumed to be {\it pure} in the sense that the true classes of each datapoint are known.  In this situation, the PDFs $P(\br)$ and $N(\br)$ can be directly constructed via modeling, e.g.\ by fitting to parameterized distributions.  While this task is often tractable when $r\in \mathbb R^m$ for $m\le 2$, it becomes more challenging in higher dimensions.  However, well-designed assays give rise to optimal domains that are simply connected, and in many cases, one of the two domains $D_n^\star$ or $D_p^\star$ is also convex.  Moreover, in typical assays, $B^\star$ often has zero measure with respect to $Q(\br;q)$ for any $q\in (0,1)$.   This suggests that directly modeling $D_n^\star$ and $D_p^\star$ is an alternative to estimating $P(\br)$ and $N(\br)$.

To motivate how we might accomplish this, recognize that the challenge of using Lemma \ref{lem:optclass} arises from needing to construct {\it two} $(m+1)$-dimensional ``manifolds,'' i.e.\ $P(\br)$ and $N(\br)$.  It would be convenient if we could use fewer and lower-dimensional objects.  In fact, we often can.  In practical diagnostic settings, is reasonable to assume that the boundary set $B^\star$ is a $(m-1)$-dimesional manifold, and we only need one such boundary to partition $\Gamma$.  Thus, we need to identify an isomorphism $U\simeq \{B,\mathcal C\}$, where $B$ is a point-set boundary separating $D_p$ and $D_n$, and $\mathcal C$ is an auxiliary convention needed to specify which ``side'' of $B$ corresponds to $D_p$ (for example).  As the next section clarifies, it is extremely convenient to represent $B$ in terms of an analytical function $\mathcal B(\br)=0$, since we may then define $\mathcal C$ via inequalities of the form $\mathcal B(\br) > 0$.  In other words, we propose to study the chain of isomorphisms 
\begin{align}
U=\{D_p,D_n\} \simeq \{B,\mathcal C \} \simeq \{\mathcal B(\br),>\}, \label{eq:isomorpha}
\end{align} 
where we could interpret $\mathcal C \simeq \, >$ to mean that $\mathcal B(\br) > 0$ defines $D_p$.  In fact, given Eqs.\ \eqref{eq:Dp} and \eqref{eq:Dn}, it is easy to establish that \eqref{eq:isomorpha} exists.  
\begin{property}[Weak Level-Set Property]
For every $q\in (0,1)$, the set $B^\star(q)=\{\br:qP(\br) = (1-q)N(\br)\}$ is either the empty set or can be expressed as $B^\star(q) = \{\br: \mathcal B^\star(\br;q)=0\}$ for some $\mathcal B^\star(\br;q)$.  Moreover, $\mathcal B^\star(q)$ can be defined so that $\mathcal B^\star(\br;q) > 0 \iff \br \in D_p^\star/B^\star(q)$ and $\mathcal B^\star(\br;q) < 0 \iff \br \in D_n^\star(q)/B^\star(q)$.   We refer to $\mathcal B^\star(\br;q)$ as a {\bf boundary function} or {\bf level-set function}. \label{post:wls}
\end{property}

The Weak Level-Set (WLS) Property is a trivial consequence of the structure of $B^\star(q)$.  When $P(\br)$ and $N(\br)$ are known, $\mathcal B^\star$ can always be constructed by inspection.  This work, however, focuses on situations for which we wish to construct $\mathcal B^\star(\br;q)$ without access to the conditional PDFs, in which case it is important to know {\it a priori} that the boundary function exists.    From a modeling standpoint, it is also useful to introduce restrictions on the WLS property.

\begin{property}[Strong Level-Set]
Let $\br \in \mathbb R^m$.  If in addition to the WLS property, the locus of points solving $\mathcal B^\star(\br;q)=0$ defines an $(m-1)$-dimensional manifold, then we say that the optimal partition $U^\star$ has the {\bf Strong Level-Set (SLS) property}.  \label{post:sls}
\end{property}

One possible complication in working directly with $\mathcal B^\star(\br;q)$ is a loss of uncertainty in the class labels.  Given $P(\br)$ and $N(\br)$, Lemma \ref{lem:optclass} guarantees an optimal way of classifying data, and for any partition $U$, $Z(\br;q,U)$ quantifies point-wise accuracy.  Yet how can we reconstruct $Z(\br;q,U)$ given only a classification boundary, or more generally guarantee that we have not lost information about $P(\br)$ and $N(\br)$?  The following lemma speaks to this issue and is central to later sections.

\begin{lemma} 
Let $\mathcal E$ be given by Eq.\ \eqref{eq:error}, and assume that up to a set of measure zero with respect to $Q(\br;q)$, there exists a partition $U^\dagger=\{D_p^\dagger,D_n^\dagger,B^\dagger\}$ with a boundary set $B^\dagger$ such that: 
\begin{itemize}
\item[i.] $D_p^\dagger \cap D_n^\dagger = D_p^\dagger \cap B^\dagger = D_n^\dagger \cap B^\dagger = \emptyset$; 
\item[ii.] $D_p^\dagger \cup D_n^\dagger \cup B^\dagger = \Gamma$; and 
\item[iii.] $\mathcal E[D_p^\dagger\cup B^\dagger_p,D_n^\dagger\cup B^\dagger_n]$ is minimized for an arbitrary partition $\{B^\dagger_p,B^\dagger_n\}$  of $B^\dagger$.
\end{itemize}
Then $U^\dagger$ is in the equivalence class of partitions given by Eqs.\ \eqref{eq:Dp} and \eqref{eq:Dn}. \label{lemma1}
\end{lemma}

\begin{proof}  Consider first the case where $B^\dagger$ is a set of measure zero, and assume the result is false.  Then there exists a set of positive measure with respect to $Q(\br;q)$ for which $D_p^\dagger$ and/or $D_n^\dagger$ differs from $D_p^\star$ and $D_n^\star$.  Consider the difference 
\begin{align}
\mathcal E[D_p^\dagger,D_n^\dagger] \!-\! \mathcal E[D_p^\star,D_n^\star]  \!=\! 
 \int_{D_p^\dagger / D_p^\star} \hspace{-8mm} (1\!-\!q)N(\br) \, {\rm d} r \!-\! \int_{D_n^\star / D_n^\dagger} \hspace{-8mm} qP(\br) \, {\rm d} r   \!+\!\int_{D_n^\dagger / D_n^\star} \hspace{-8mm} qP(\br) \, {\rm d} r \!-\! \int_{D_p^\star / D_p^\dagger} \hspace{-8mm} (1\!-\!q)N(\br) \, {\rm d} r. \nonumber
\end{align}
Note also that $D_p^\dagger / D_p^\star = D_n^\star / D_n^\dagger$ and $D_n^\dagger / D_n^\star = D_p^\star / D_p^\dagger$, since $U^\dagger$ and $U^\star$ are both partitions of $\Gamma$.  Thus
\begin{align}
\mathcal E[D_p^\dagger,D_n^\dagger] - \mathcal E[D_p^\star,D_n^\star]  = \int_{D_n^\star / D_n^\dagger} \hspace{-8mm} (1-q)N(\br) - q P(\br)\, {\rm d} r   +\int_{D_p^\star / D_p^\dagger} \hspace{-8mm} qP(\br) - (1-q)N(\br) > 0,\nonumber
\end{align}
where the last inequality is a consequence of the definitions of $D_p^\star$ and $D_n^\star$.  Thus, we find $\mathcal E[D_p^\star,D_n^\star] < \mathcal E[D_p^\dagger,D_n^\dagger]$, which is a contradiction.  

The case for which $B^\dagger$ has positive measure follows by analogous arguments, taking into account that this set can be arbitrarily partitioned between the positive and negative classification domains.  
\end{proof}

Lemma \ref{lemma1} is the converse of the Lemma \ref{lem:optclass}.  Lemma \ref{lemma1} states that the equivalence class of partitions minimizing $\mathcal E(U;q)$ is the same as the equivalence class  $U^\star$ defined explicitly in terms of $P(\br)$ and $N(\br)$.  Perhaps unsurprisingly, this implies that any partition $U^\dagger$ minimizing $\mathcal E$ is as good as invoking Lemma \ref{lem:optclass}, no matter how we find $U^\dagger$.  More importantly, however, $D_p^\dagger$ and $D_n^\dagger$ must preserve inequalities of the form $qP(\br) > (1-q)N(\br)$ etc., even though the conditional PDFs need not be used to construct these sets.  In this context, the importance of Eq.\ \eqref{eq:isomorpha} is difficult to overstate.  In Sec.\ \ref{sec:UQ}, we show that the ability to construct $B^\star(q)$ is central to extracting uncertainties in class labels.  

However, it is necessary to first develop a practical method for constructing $B^\star(q)$.  To this end, the following definition and postulates are useful.
\begin{definition}[Boundary Classifier]
Let $\mathcal B(\br;\phi(q))$ depend on some parameters $\phi(q)$.  We define $\hat C_{\mathcal B}(\br;q) = \mathcal H(\mathcal B(\br;\phi(q)))$ to be a {\bf boundary classifier}, where $\mathcal H$ is the Heaviside step-function and $\mathcal H(0)$ can arbitrarily be assigned the value $0$ or $1$. \label{def:bclass}
\end{definition}

Clearly the locus of points $\mathcal B(\br;\phi(q))=0$ plays a special role in our analysis, and our aim is to understand in what sense $\mathcal B(\br;\phi(q)) = \mathcal B^\star(\br;q)$.  A key challenge, however, is that the locus of points satisfying $\mathcal B(\br;\phi(q))=0$ can differ from $B^\star(q)$ on a set of measure zero, which is especially problematic if the latter is itself a set of measure zero!
The resolution to this problem is to postulate that such differences do not violate the structure of Eqs.\ \eqref{eq:Dp} and \eqref{eq:Dn}.

\begin{definition}
We say that $U^\dagger(q)=\{D_p^\dagger,D_n^\dagger,B^\dagger\}$ satisfies the {\bf optimal partition convention} if $U^\dagger(q)$: (i) is in the equivalence class minimizing $\mathcal E(U;q)$; and (ii) ${\br \in D_p^\dagger(q) \iff qP(\br) \ge (1-q)N(\br)}$ and ${\br \in D_n^\dagger(q) \iff qP(\br) \le (1-q)N(\br)}$.  \label{def:optconv}
\end{definition}
 
\begin{remark}
Definition \ref{def:optconv} implies that all $\br$ are classified correctly \textit{pointwise}.  This is needed to avoid situations in which, for example, a set of measure zero satisfying $qP(\br) > (1-q)N(\br)$ is assigned to $D_n^\star(q)$.  While such partitions do not affect the average classification accuracy, they clearly affect pointwise UQ.  Note that $qP(\br) = (1-q)N(\br) \nRightarrow \br \in B^\dagger(q)$.   
\end{remark}

\begin{postulate}[Weak Level Set]
A classifier minimizing $\mathcal E(U;q)$ yields a partition that satisfies the optimal partition convention.  
\end{postulate}

The WLS postulate does not guarantee that a boundary function identifies all $\br \in B^\star(q)$.   From a modeling standpoint, this suggests a stronger postulate.

\begin{postulate}[Strong Level-Set]
A boundary classifier minimizing $\mathcal E(U;q)$ satisfies the WLS postulate and $\mathcal B(\br;\phi(q))=\mathcal B^\star(\br;q)$ for all $q\in (0,1)$.  
\end{postulate}


\section{Classification Without PDFs}
\label{sec:classification}

\subsection{General Formulation of the Classification Method}
\label{subsec:gentheory}

Lemma \ref{lemma1} and the SLS postulate do not tell us how to find $B^\star$; one cannot avoid a modeling choice.  Assume under the SLS postulate that an arbitrary boundary $B$ is defined as the locus of points satisfying a general nonlinear equation  
\begin{align}
\mathcal B(\br;\phi)=0 \label{eq:gennonlinear}
\end{align}
where $\phi$ are parameters that determine the shape of the boundary.  Consistent with Def.\ \ref{def:bclass}, define the domains $D_p$ and $D_n$ via the inequalities $\mathcal B(\br;\phi) > 0$ and $\mathcal B(\br;\phi) < 0$, respectively.  Here we present the general framework and considerations for estimating $\mathcal B(\br;\phi)$; specific models used in this work are discussed in Sec.\ \ref{subsec:quadratic}.  

\begin{remark}
We assume that $\B$ is twice differentiable in $\phi$, which permits the use of Newton's Method and related approaches in subsequent calculations \cite{Nocedal,NA}.
\end{remark}

To determine the $\phi$, we require a rigorous concept of {\it training data}.  In particular:
\begin{definition}[Pure Training Datapoint]
Let $\omega$ be a sample point, and let $C(\omega)$ be the true class of $\omega$.  The pair $(C(\omega),\br(\omega))$ is a {\bf pure training data point.}  
\end{definition}

\begin{definition}[Pure Training Population]
Assume that a population has two classes indexed by $k$, so that $k\in \{0,1 \}$.  Let $n_k$ be positive integers, and let $\omega_{k,i}$ be distinct sample points in $\Omega$.  Then the collection of sets indexed by $k$ and defined via
\begin{align}
\Pi=\big\{\Pi_k \big\},  && 
\Pi_k = \{\br_{i}\!:\!\br_i=\br(\omega_{k,i}),i\in \{1,...,n_k\},C(\omega_{k,i})=k \}_k
\end{align}
is a {\bf training population}.  We refer to the sets $\Pi_k$ as {\it training sets} associated with population $k$.  For brevity, we denote the $i$th element of the $k$th class by symbol $\br_{k,i}$.\label{def:pure}
\end{definition}


Training data and classifiers are ultimately used to deduce the properties of test population.  We are usually only given a finite number of samples from such a population, which motivates the following definition:
\begin{definition}[Empirical Test Population]
Let $\Omega$ be a sample space and $C(\omega)$ a binary random variable for $\omega \in \Omega$.  Let $q_s\in [0,1]$, and the $\omega_i$ are {\it iid}.    We define 
\begin{align}
\Psi(q_s,s)=\left\{ \br(\omega_i):i\in \{1,...,s\}, \sum_{i=1}^s\frac{C(\omega_i)}{s} = q_s \right\} \label{eq:impureset} 
\end{align}
to be an {\bf empirical test population} with $s$ samples and prevalence $q_s$.  When there is no risk of confusion, we also refer to $\Psi$ more simply as a test population.
\end{definition}
\begin{remark}
It is important to observe that the prevalence $q_s$ of an empirical test population may not be identical to the $q$ appearing in Eq.\ \eqref{eq:qdef}.  The latter may take any value in $[0,1]$, whereas $q_s$ is limited to rational numbers of the form $n_1/(n_0+n_1)$.  Only in the limit $n_0+n_1\to \infty$ do we expect the two numbers to coincide.  
\end{remark}

Our goal is to use $\Pi$ to find the boundary function that minimizes the empirical classification error.  This can then be used to classify the samples in a test population $\Psi(q,s)$.  However, in light of Sec.\ \ref{subsec:bg}, we need a method for estimating $q$, since $P(\br)$ and $N(\br)$ are no longer assumed to be known.  This motivates the following.
\begin{definition}[Empirical Prevalence Estimate]
Let $\Pi$ be a pure training population, and let $D\subset \Gamma$ be any subdomain chosen such that the empirical estimates
\begin{align}
\tilde N_D = \frac{1}{n_0}\sum_{j=1}^{n_0} \mathbb I(\br_{0,j} \in D), &&
\tilde P_D = \frac{1}{n_1}\sum_{j=1}^{n_1} \mathbb I(\br_{1,j} \in D)
\end{align}
are not equal, where $\br_{i,j}\in \Pi_i$.  Moreover, assume that $P_D \ne N_D$, and define $\tilde Q_D$ according to Eq.\ \eqref{eq:prevest} in terms of the empirical test population.  Then we call 
\begin{align}
\hat q = \frac{\tilde Q_D - \tilde N_D}{\tilde P_D - \tilde N_D} \label{eq:prevest_emp}
\end{align}
the {\bf empirical prevalence estimate} of $q$.  \label{constr:prev_emp}
\end{definition}

\begin{lemma}
Assume there exists a positive-measure set $D$ for which $P(\br)>N(\br)$, and let $q \in (0,1)$.  Let $\hat q$ denote the empirical prevalence estimate, and assume that the number of training points $n_k \ge n$ for all classes $k$ and some $n > 0$; see Def.\ \ref{def:pure}.  Let $s \ge n$ denote the number of test points. Then the following statements are true:
\begin{itemize}
\item[i.] $\hat q$ exists and is positive, both almost surely as $n\to\infty$;
\item[ii.] $\hat q$ converges to Eq.\ \eqref{eq:prevest} at a rate of $n^{-1/2}$ as $n\to \infty$;
\item[iii.] $\hat q$ is an asymptotically unbiased estimator of $q$.
\end{itemize}\label{lem:prev_emp}
\end{lemma}

\begin{remark}
In the proof of Lemma \ref{lem:prev_emp}, we temporarily assume the sample space corresponds to a set of countably infinite measurements on individuals $\omega\in \Omega$.  Thus, the sample space referenced below is distinct from $\Omega$ described in Sec.\ \ref{sec:setting}.
\end{remark}

\begin{proof}
To prove i.\ note first that $P_D > N_D$.  Let $\epsilon = (P_D - N_D)/4$.  Clearly if $\tilde P_D > P_D - \epsilon$ and $\tilde N_D < N_D + \epsilon$, then $\hat q$ exists.  Thus, we seek to show that these inequalities hold almost surely.

Consider first the case for which $n_0=n_1=s=n$, and recognize that $n\tilde P_D$ and $n\tilde N_D$ are binomial random variables whose means are $nP_D$ and $nN_D$, respectively.  Let $\mathcal A$ be the event for which $P_D - N_D \le 0$, $\mathcal A_P$ be the event for which $\tilde P_D \le P_D - \epsilon$, and $\mathcal A_N$ be the event for which $\tilde N_D \ge N_D + \epsilon$.  Clearly $\mathcal A \subset \mathcal A_P \cup \mathcal A_N$.  This implies 
\begin{align}
\Pr[\mathcal A] \le \Pr[\mathcal A_P] + \Pr[\mathcal A_N]. \label{eq:measure_sep}
\end{align}

Consider next $\Pr[\mathcal A_P]$, which is the probability of at most $k \le nP_D - n\epsilon$ successes given $n$ trials.  By Hoeffding's inequality \cite{Hoeffding},
\begin{align}
\Pr[\mathcal A_P] &  \le \exp\left[-2n \left(P_D - \frac{nP_D - n\epsilon}{n} \right) \right] = e^{-2n\epsilon}.
\end{align}
Similarly, one finds $\Pr[\mathcal A_N] \le \exp(-2n\epsilon)$.  Inequality \eqref{eq:measure_sep} therefore implies 
\begin{align}
\Pr[\mathcal A] \le 2\exp(-2n\epsilon).
\end{align}

Consider next a sequence of events $\mathcal A_n$ for which $\tilde P_D - \tilde N_D \le 0$.  One finds
\begin{align}
\sum_{n=1}^\infty \Pr[\mathcal A_n] \le \sum_{n=1}^\infty 2\exp(-2n\epsilon) < \infty.
\end{align}
Therefore, by the Borel-Cantelli lemma \cite{Borel,Cantelli}, the event for which $\tilde P_D - \tilde N_D \le 0$ occurs infinitely often has measure zero.  That is there exists almost surely a finite value of $n$ for which $\tilde P_D - \tilde N_D$ is guaranteed to be positive.  As $P_D - N_D$ is positive by assumption, the $\hat q$ so constructed exists and is finite.  A similar argument applied to the numerator of Eq.\ \eqref{eq:prevest_emp} proves that there exists an $n$   for which $\tilde Q_D - \tilde N_D>0$.  

To prove the result for the case in which the number of training and test points are not equal, let $n$ be the minimum among $n_0$, $n_1$, and $s$.  Then construct $\hat q$ using $n$ points from each training and test set.  This reduces the problem to the previous case.   

Assertion ii.\ is a consequence of the Central Limit Theorem.  That is, as $n\to \infty$, $\tilde P_D$ and $\tilde N_D$ converge to normal random variables whose variances go as $P_D(1-P_D)/n$ and $N_D(1-N_D)/n$.  As $\hat q$ is guaranteed to exist for some finite value of $n$, a Taylor expansion of the denominator in Eq.\ \eqref{eq:prevest_emp} yields the desired rate of convergence; see also the Delta method \cite{Doob}.  Note that assertion iii.\ follows directly from ii.  
\end{proof}

\begin{remark}
Lemma \ref{lem:prev_emp} only states that \textit{with sufficient data} and a reasonable choice of $D$, $\hat q$ exists and is positive definite.   In such cases, it should be possible to estimate the $D$ used in Lemma \ref{lem:prev_emp} by computing the $q=1/2$ classification domain via the method presented at the end of this section.  Based on the Central Limit Theorem, is also straightforward to show that $n \gg [q(P_D - N_D)]^{-2}$ characterizes the number of datapoints for which we can expect the estimator $\hat q$ to be well-behaved with high probability.  Thus, prevalence values close to $0$ or $1$ require a significant amount of data to resolve, as do measurements for which $P_D \approx N_D$.  
\end{remark}

Given a method for estimating the prevalence, we are now in a position to define the objective function whose minimum is our classification boundary of interest.  
\begin{definition}[Binary Empirical Error]
Let $\Pi$ be a pure training population with $2$ classes.  Then the {\bf prevalence-weighted empirical error} $\mathcal E_e\left[U;\Pi,q \right]$ is
\begin{align}
\mathcal E_e\left[U;\Pi,q \right] &= \frac{q}{n_p}\sum_{i=1}^{n_p}\mathbb I(\br_{p,i} \in D_n) + \frac{1-q}{n_n}\sum_{i=1}^{n_n}\mathbb I(\br_{n,i} \in D_p).\label{eq:emperror}
\end{align}
\end{definition}

\begin{remark}
The prevalence $q$ appearing in Eq.\ \eqref{eq:emperror} is either associated with a test population or left as a free parameter.  We do not generally fix {${q=n_p/(n_p+n_n)}$}, which is the prevalence of the training population.  
\end{remark}

We justify Eq.\ \eqref{eq:emperror} by observing that it is a Monte Carlo estimate of the expected error \eqref{eq:error} \cite{montecarlo}.  By Def.\ \ref{def:bclass}, it is also easy to see that the empirical error can be re-expressed in terms of the Heaviside step function $\mathcal H(x)$ via
\begin{align}
\mathcal E_e\left[\phi;\Pi,q \right] &= \frac{q}{n_p}\sum_{i=1}^{n_p}\left[1-\mathcal H(\mathcal B(\br_{p,i};\phi))\right]  + \frac{1-q}{n_n}\sum_{i=1}^{n_n}\mathcal H(\mathcal B(\br_{n,i};\phi)), \label{eq:reemp}
\end{align}
where, in a slight abuse of notation, we use the same symbol $\mathcal E_e$ in Eqs.\ \eqref{eq:emperror} and \eqref{eq:reemp}.  Used as an objective function for determining $\phi$, however, Eq.\ \eqref{eq:reemp} can yield an uncountably infinite number of partitions $U$ with the same empirical error.  Perhaps worse, for finite $n_n$ and $n_p$, this objective function has a gradient (with respect to $\phi$) that is everywhere zero or infinite.  This motivates the following modification.
\begin{definition}[Regularized Empirical Error]
Let $\Pi$ be a pure training population with two classes, and let $H(x)$ be a strictly monotone increasing sigmoid function for which: the domain is $\mathbb R$; the range is $[0,1]$; and $H(0)=1/2$.  Then the {\bf regularized empirical error} $\mathcal L\left[\phi;\Pi,q,\sigma^2 \right]$ is
\begin{align}
\mathcal L \left[\phi;\Pi,q,\sigma^2 \right] &= \frac{q}{n_p}\sum_{j=1}^{n_p}\left[ 1-H \left( \frac{\mathcal B(\br_{p,j};\phi)}{\sigma^2} \right) \right]  + \frac{1-q}{n_n} \sum_{j=1}^{n_n} H \left( \frac{\mathcal B(\br_{n,j};\phi)}{\sigma^2} \right). \label{eq:empiricalobjective}
\end{align}
\end{definition}

\begin{remark} In this manuscript, we take $H(x)$ to be $H(x) = 0.5[1+{\rm tanh}(x)]$. \end{remark}

  To understand the importance the regularization parameter $\sigma^2$ in Def.\ \eqref{eq:empiricalobjective}, note that in the limit $\sigma\to 0$,  Eq.\ \eqref{eq:empiricalobjective} reverts to the empirical classification error of Eq.\ \eqref{eq:emperror}.  Positive values of $\sigma^2$ therefore play two related roles.  First, $\sigma$ renders the objective function continuous and differentiable in $\phi$, provided $\mathcal B$ has these properties.  Second, for vanishing values of $\sigma$, the objective function is clearly non-convex and may have many local minima.  Informally speaking, large values of $\sigma$ help to ``convexify'' $\mathcal L$, in essence expanding the radius of convergence of the optimization algorithm.  However, this begs the question: how does one choose $\sigma$?

We address this conundrum by treating $\sigma$ as a homotopy parameter; see Refs.\ \cite{Homotopy1,Homotopy2,Homotopy3,Homotopy4,Homotopy5} for background and related approaches.  Our goal is to reach the limit $\sigma \to 0$ while stabilizing the optimization, for which the latter requires $\sigma$ to be large.  This motivates the following definition.
\begin{definition}
Let $\{\sigma_j^2\}_{j=0}^{s}$ be a monotone decreasing sequence.  Let $\mathcal M$ be an optimization algorithm that maps the pair $\mathcal L$ and $\phi$ back onto the space of $\phi$, where $\mathcal M$ approximates the $\argmin$ operator.  Given an initial $\phi_0^\star$, we define the $\phi_{j+1}^\star$th {\bf homotopy classifier} to be the solution to the optimization problem  
\begin{align}
\phi_{j+1}^\star = \mathcal M [\mathcal L(\phi;\Pi,q,\sigma_j^2) ; \phi_j^\star], \label{eq:homotopy}
\end{align}
where $\phi_j^\star$ is used as initial point associated with the $(j+1)$th optimization problem. \label{def:homotopy}
\end{definition}

\begin{remark}
Observe that we may estimate any level set $\mathcal B(\br;\phi^\star(q))$ by varying $q$ in Eq.\ \eqref{eq:homotopy}.  However, the level-sets so constructed can potentially cross, which would lead to contradictions in the definitions of the underlying PDFs if interpreting $\mathcal B(\br;\phi^\star(q))$ as $\mathcal B^\star(\br;q)$.  Section \ref{sec:UQ} addresses this issue.
\end{remark}

The choices of $\{\sigma^2_j\}_{j=0}^{s}$ and $\phi_0^\star$ in Def.\ \ref{def:homotopy} are problem specific; in Sec.\ \ref{subsec:sars} we propose an algorithm for determining them.  However, taking $\sigma^2$ to be the minimum characteristic distance between datapoints should approximate the limit $\sigma \to 0$, provided $\mathcal B$ remains order one for all values of $\br$.  Moreover, we are not guaranteed a unique minimizer in the limit $\sigma \to 0$.  Herein, we never consider $\sigma^2 < 10^{-8}$.  

We are now in a position to construct our classifier for {\it test data}.

\begin{construction}[Test Data Homotopy Classifier]
Let $\Pi$ and $\Psi(q',s)$ denote training and test populations, and assume that the number of datapoints in both sets is sufficient to guarantee that $\hat q$ exists.  Let $\phi_0^\star$ denote an initial point and $\{\sigma_j^2\}_{1}^J$ be a sequence of homotopy parameters.  

To construct a classifier for the test data, set $q=1/2$ in Def.\ \ref{def:homotopy} to compute the $(J+1)$th homotopy classifier $\phi^\star_{J+1}$.  This yields a boundary function $\mathcal B(\br; \phi^\star_{J+1})$, which approximates the set $D$ of Lemma \ref{lem:prev_emp}.  Moreover, by convention, any point $\br_{p,i}$ for which $\mathcal B(\br_{p,i};\phi^\star_{J+1}) > 0$ contributes to the average $\tilde P_D$, and similarly for points from the negative training set and test population.  Using these estimates, construct $\hat q$ using Eq.\ \eqref{eq:prevest_emp}.  This quantity is an estimator of $q'$ in the sense of Lemma \ref{lem:prev_emp}.

Next set $q=\hat q$ in Def.\ \ref{def:homotopy} and compute the sequence of homotopy classifiers, which yields a new estimate of $\phi^\star_{J+1}$.  This quantity now estimates the optimal classification boundary associated with $q'$.  As in the previous set, use the new $\phi^\star_{J+1}$ to construct the classification boundary $\mathcal B(\br; \phi^\star_{J+1})$. \hfill {\small Q.E.F.} \label{constr:homotopy}
\end{construction}

\subsection{Quadratic Approximation}
\label{subsec:quadratic}

To illustrate the concepts developed in the previous section, we consider an important model of $\mathcal B$.

\begin{example}[Quadric Boundary Sets]
Assume that an arbitrary boundary set in $m$-dimensions can be represented as a generalized quadratic equation of the form
\begin{align}
B=\left\{\br:(\br,1)\A \binom{\br}{1} = 0 \right\} \label{eq:quadric}
\end{align} 
where $\A$ is an $(m+1)\times (m+1)$ matrix (for $\br\in \mathbb R^m$) of free parameters that depend on $q$.  That is,  identify $\phi$ as the matrix $\A$.  Letting $\chi^{\rm T}=(\br,1)$, this corresponds to the boundary function
\begin{align}
\mathcal B(\br;\A) = \chi^\T \A \chi. \label{eq:quadb}
\end{align}
Without loss of generality, supplement this with the symmetry constraint
\begin{align}
A_{i,j}=A_{j,i}. \label{eq:symmetry}
\end{align} \label{ex:quadB}
\end{example}

While the quadratic model is relatively simple, it supports efficient computation and generalization through vectorization.  This example also illustrates auxiliary assumptions and regularization needed to stabilize the analysis.  In particular, note that $\mathcal L$ as defined by Eq.\ \eqref{eq:empiricalobjective} has a plethora of connected sets when expressed in terms of $\A$.  Thus, the minimizer is not unique.    Among these is a duplication of the off-diagonal parameters in $\A$.  To see this, recognize that a quadratic in $m$ variables can be defined in terms of $(m+1)(m+2)/2$ parameters, although $\A$ has $(m+1)^2$.  The symmetry constraint given by Eq.\ \eqref{eq:symmetry} resolves this issue.

A more troublesome connected set arises from the interplay between $\sigma$ and $\A$.  A general quadratic is more simply defined in terms of only $(m+1)(m+2)/2-1$ parameters, one less than fixed by a symmetry constraint.  The extra parameter amounts to a change of scale and is associated with the fact that we may divide all coefficients in a polynomial by any one of them.  Thus, without further regularization, minimization of Eq.\ \eqref{eq:empiricalobjective} yields a coefficient matrix $\A$ whose elements tend to diverge in a way that undoes the regularization associated with $\sigma^2$.  This arises from the fact that large values of $\sigma$ ``blur-out'' the individual data points, so that those near the boundary increase the objective function.  A solution to this problem is to impose additional regularization that fixes the scale.
\begin{definition}[Scale-Regularized Objective]
Let the quantity $\mathcal L_{\rm scale}$ be defined as
\begin{align}
\mathcal L_{\rm scale} = \left( \sum_{j=1}^{n_p}\frac{\mathcal B(\br_{p,j},\phi)^2}{n_p} +  \sum_{j=1}^{n_n}\frac{\mathcal B(\br_{n,j},\phi)^2}{n_n} - 1 \right)^2.  \label{eq:regscale}
\end{align}
Then we define the {\bf scale-regularized objective function} $\mathcal L_{sr}$ to be
\begin{align}
\mathcal L_{sr} = \mathcal L  + \mathcal L_{\rm scale}. \label{eq:objectivesum}
\end{align} \label{def:scalereg}
\end{definition}
\begin{remark}
The boundary set is the locus of points for which $\mathcal B^\star(\br;q)=0$.  Thus, Eq.\ \eqref{eq:regscale} requires that the sum of ``variances'' of the $\mathcal B(\br,\phi)$ (relative to $0$) be $\mathcal O(1)$.  This ensures that the argument of $H(x)$ in Eq.\ \eqref{eq:empiricalobjective} remains $\mathcal O(1/\sigma^2)$, so that $\sigma$ controls the regularization scale associated with the homotopy optimization.
\end{remark}

\section{Validation}
\label{sec:validation}

\subsection{Synthetic Data}
\label{subsec:syndata}
We test our analysis using 2D synthetic datasets, beginning  with a simple consistency check to ensure that the analysis converges when the classification boundaries are quadratic.  
\begin{example}[Parabolic-Gaussian Distributions]
Consider the conditional PDFs
\begin{align}
N(x,y)=g(x)g(y-x^2+3), && P(x,y)=g(x)g(y-x^2),  \label{eq:PandN}
\end{align}
where $g(x)$ is the PDF of the standard normal distribution (i.e.\ with mean zero and unit variance); see Figure \ref{fig:syndat}.  In light of Eqs.\ \eqref{eq:Dp} and \eqref{eq:Dn}, consider the ratio
\begin{align}
0 < \frac{N(x,y)}{P(x,y)} = \frac{g(y-x^2+3)}{g(y-x^2)} < \infty.
\end{align}
Clearly the locus of points satisfying the equation
\begin{align}
\frac{N(x,y)}{P(x,y)} = \frac{q}{1-q}
\end{align}
is a quadratic of the form $c=x^2-y$.
In particular, for $q=0.5$, the minimum-error classification boundary is given by
\begin{align}
\A^\star = \begin{bmatrix}
1 & 0 & 0 \\ 0 & 0 & -0.5 \\ 
0 & -0.5 & -1.5
\end{bmatrix}.\label{eq:obvmat}
\end{align}
\end{example}

\begin{remark}
It is well known that multivariate Gaussian distributions yield optimal classification boundaries that are quadrics.  In fact, this is the basis for discriminant analyses \cite{Venables}.  Equations \eqref{eq:PandN} are therefore notable insofar as they are {\it not} multivariate Gaussian distributions, as illustrated in Fig.\ \ref{fig:syndat}.  We return to this point in Sec.\ \ref{sec:discussion} when pursuing a deeper comparison with ML techniques.
\end{remark}

\begin{figure}\begin{center}
\includegraphics[width=12.5cm]{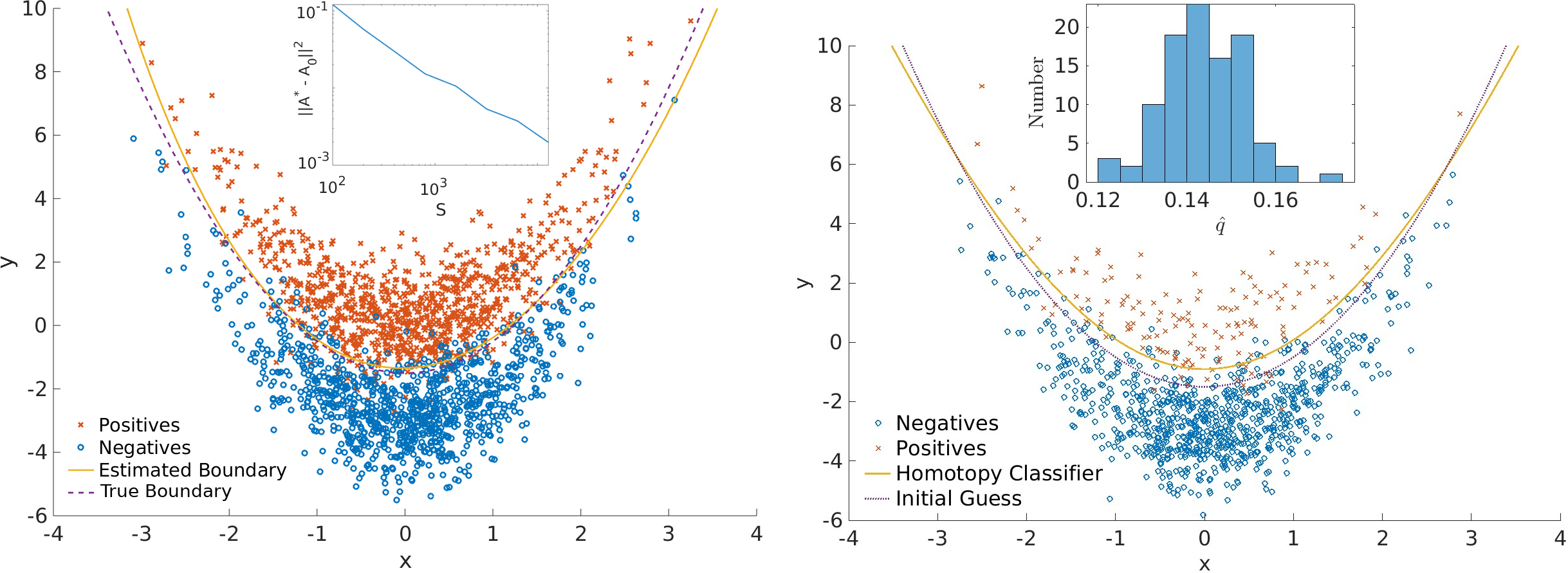}\end{center}\caption{{\it Left:}  Example of 2000 dimensionless, synthetic data points and analysis thereof.  Positive and negative data-points were generated according to the distributions given by Eqs.\ \eqref{eq:PandN}.  The estimated classification boundary is given by solving Eq.\ \eqref{eq:homotopy}.  The inset shows the average Frobenius norm $|| \A^\star - \A_{s+1}^\star||^2$ as a function of the number of sample points $\mathcal S$.  Note that the norm displays an approximate $1/S$ behavior consistent with convergence in mean-square.  {\it Right:} Illustration of the Homotopy Classifier; see Construction \ref{constr:homotopy}. Data was generated from the same distributions as in the left plot.  However, the true prevalence was $q=0.15$.  While the true classes of the datapoints are indicated in the figure, all data was combined and treated as test data for the purposes of this example.  All optimizations were done using a training population constructed in the same manner as the left plot.  Moreover, the same initial conditions and homotopy parameters were used to determine $\hat q$ (estimated to be $0.1525$ for the data shown) for the test population, after which the classification boundary was optimized using the empirical prevalence estimate.  The inset shows a histogram of $\hat q$ values associated with repeating the this numerical experiment 100 times.  }\label{fig:syndat}
\end{figure}

The left panel of Fig.\ \ref{fig:syndat} illustrates classification of synthetic data generated by Eqs.\ \eqref{eq:PandN} and analyzed by computing the homotopy classifier of the scale-regularized objective function with a quadratic boundary; see Eq.\ \eqref{eq:quadb} and Defs.\ \ref{def:homotopy} and \ref{def:scalereg}.  Here we set $q=0.5$ and fix a number $\mathcal S$ of total synthetic samples, with $n_n=n_p=\mathcal S/2$.  In the limit that $\mathcal S\to\infty$, we anticipate (but do not prove) that the optimal classification boundary should converge to Eq.\ \eqref{eq:obvmat}.  To test this, we consider values of $\mathcal S$ defined by $\mathcal S=200 \times 2^k$ for $k=0,1,...,7$.  For each of these values, we generate $M=100$ synthetic datasets drawn from the distributions given by Eqs.\ \eqref{eq:PandN} and evaluate the Frobenius norm
\begin{align}
|| \A^\star - \A_{s+1}^\star||^2 = \sum_{i}\sum_{j} [\A^\star - \A_{s+1}^\star]_{i,j}^2,
\end{align}
where the initial condition for the optimization is set to be $\A_0^\star = \A^\star$, and $\A_{s+1}^\star$ is determined by solving Eq.\ \eqref{eq:homotopy}.  We also fix the sequence of $\sigma^2$ values to be $\{\sigma^2_j=10^{-j}\}_{j=1}^5$.  The Frobenius norm tests for a degree of mean-squared convergence in $\A_{s+1}^\star$ as a function of the size $\mathcal S$ of the empirical distribution.  The inset to the left plot \ref{fig:syndat} shows $\mathcal S$ versus $|| \A^\star - \A_{s+1}^\star||^2$, which demonstrates the characteristic $1/\mathcal S$ rate of convergence.  While this example idealizes certain aspects of diagnostic classification, it suggests that the homotopy classifier (Def.\ \ref{def:homotopy}) can converge to the true solution, or at least an approximation thereof, under some circumstances.

The right panel of Fig. \ref{fig:syndat} shows an example wherein the same distributions were used to generate test data with a prevalence of $q=0.15$.  A positive and negative training set, each comprised of 1000 datapoints, were used as inputs to the scale-regularlized objective function. Construction \ref{constr:homotopy} was then used to estimate the prevalence of the test population and determine the classifier.  The inset shows that with 1000 datapoints in the test population, the estimate $\hat q$ has a characteristic standard deviation of about $0.015$.

\subsection{A SARS-CoV-2 ELISA Assay}
\label{subsec:sars}

\begin{example}[Detection of Previous Infection]\label{ex:sars}
A more realistic example arises in the context of a research-use-only COVID-19 enzyme-linked immunosorbent assay (ELISA). Since early 2020, such diagnostic tests have been used to detect the presence of anti-SARS-CoV-2 antibodies, which indicate a previous infection \cite{genantibody}.

Here we consider a diagnostic test developed in Ref.\ \cite{Raquel1} that detects immunoglobulin G (IgG) antibodies binding to the receptor binding domain (RBD), spike (S), and nucleocapsid (N) of the SARS-CoV-2 virus among pandemic, pre-vaccine era samples \citep{Raquel1}.  It is meaningful to identify positive samples in terms N and RBD measurements, since high titers of these antibodies are typically associated with a strong immune response to COVID-19 infection.  Training data is shown in Fig.\ \ref{fig:sarsdata}, along with the conics associated with the initial guess $\A_0^\star$ and final estimate $\A_{s+1}^\star$.  Note that Fig.\ \ref{fig:sarsdata} shows the data after transforming each dimension according to $x_i \to \ln(.01 + x_i-x_{i,{\rm min}})$, where $x_i$ is the $i$th coordinate associated with a measurement $\br$, and $x_{i,{\rm min}}$ is the minimum value of the $i$th coordinate taken across all of the negative measurements.  The $10^{-2}$ in the argument of the logarithm is a modeling choice that amounts to setting a finite origin for the data on a logarithmic scale.  
\end{example}

\begin{figure}
\includegraphics[width=13cm]{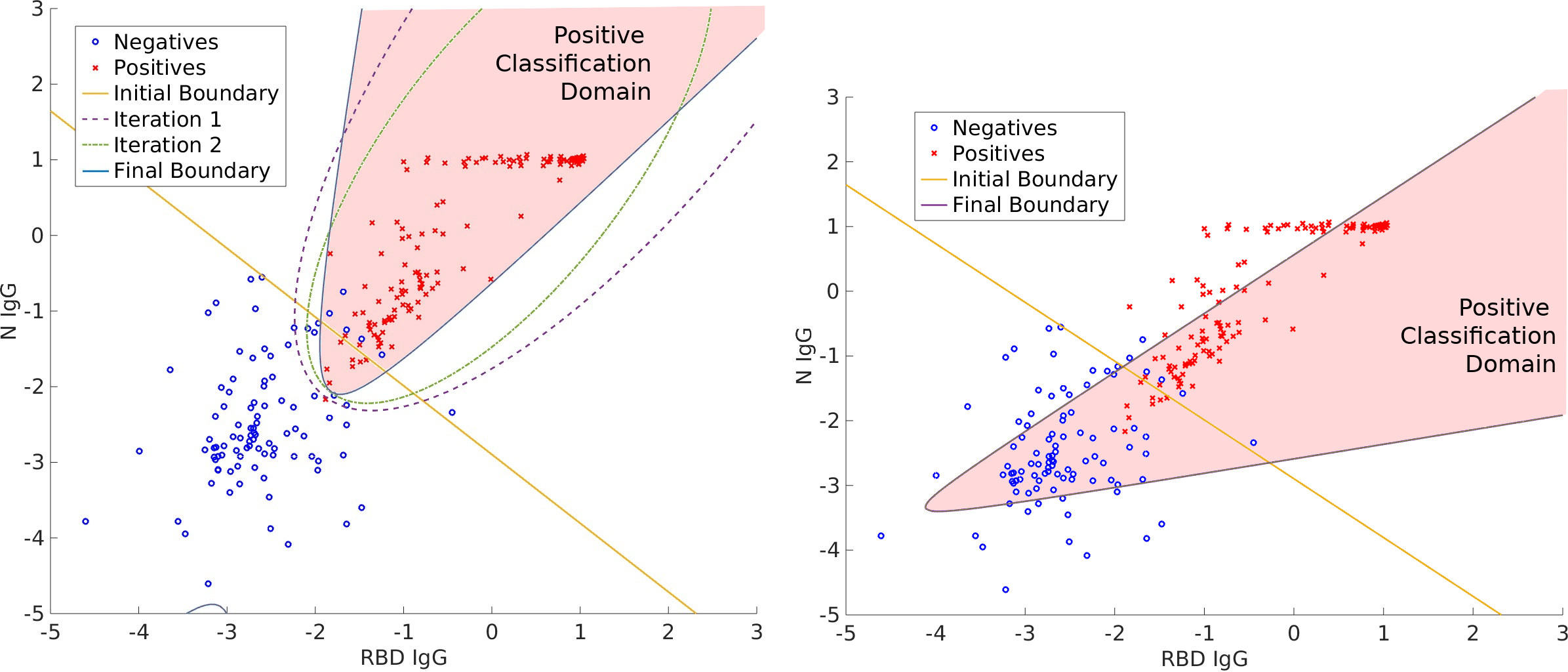}\caption{{\it Left:} Classification of 2D SARS-CoV-2 training data for $q=268/460$, the true prevalence of the dataset.  Negative samples were collected before the pandemic.  Positive samples were collected before the release of SARS-CoV-2 vaccines; thus, RBD is still a meaningful determinant of a positive sample.   The initial classification boundary is a straight line determined according to Eq.\ \eqref{eq:ao}.  The boundaries associated with the first two iterations of Eq.\ \eqref{eq:homotopy} are shown in dotted and dash-dot lines.  By the third iteration, the homotopy method has converged to a solution.  The corresponding value of $\sigma^2$ corresponds to $|\nu|/1000$.  {\it Right:} Example of optimization without using the homotopy method.  The initial guess of the optimization was the same as in the left plot.  We directly set $\sigma^2=10^{-8}$ and performed a single optimization via Eq.\ \eqref{eq:homotopy}.  The resulting positive classification domain poorly separates positive and negative populations.}\label{fig:sarsdata}
\end{figure}

In contrast to the previous example, identifying a reasonable initial guess $\A_0^\star$ requires care.  Intuitively it makes sense to define $\A_0^\star$ as the hyperplane (or in 2D, the line) that ``best'' separates the populations in some appropriate sense.  The following construction provides a reasonable definition.
\begin{construction}[Weighted Hyperplane Classifier]
Given a binary training population, compute the empirical means $\mu_p$ and $\mu_n$ of the positive and negative samples.  The vector $\nu=\mu_p - \mu_n$ separating these defines a direction, which we take to be perpendicular to the hyperplane of interest.  We need only identify a suitable origin $\nu_0$ to fully specify the hyperplane.  

To find $\nu_0$, compare the relative sizes of the distributions $P(\br)$ and $N(\br)$ in the direction of $\nu$.  Specifically, consider the empirical covariance matrices for 
\begin{align}
\Xi_k &= \frac{1}{n_k-1} \sum_{j=1}^{n_k} (\br_{k,j} - \mu_k)(\br_{k,j} - \mu_k)^{\rm T} ,
\end{align}
for $k\in\{0,1\}$ (or $k=n$ or $k=p$); see Sec.\ \ref{subsec:notation}.  
Define the weights
\begin{align}
w_k = \sqrt{\nu^{\rm T} \Xi_k \nu}  &&
w_n = \sqrt{\nu^{\rm T} \Xi_n \nu}
\end{align}
which quantify the sizes of $P(\br)$ and $N(\br)$ in the direction of $\nu$. Then define $\nu_0$ as
\begin{align}
\nu_0 = \mu_n + \frac{w_n}{w_n+w_p}\nu
\end{align}
which is  a weighted center between the two distributions in the direction of $\nu$.  

The matrix form of $\A_0^\star$ follows immediately once $\nu_0$ and $\nu$ are known.  Recall that a hyperplane is defined by the equation $(\br-\nu_0)\cdot \nu =0$.  This implies that the symmetrized version of $\A_0^\star$ is given in block form by
\begin{align}
\A_0^\star=\begin{bmatrix}
{\boldsymbol 0} & \nu/2 \\
\nu^{\rm T}/2 & -\nu \cdot \nu_0
\end{bmatrix} \label{eq:ao}
\end{align}
where ${\boldsymbol 0}$ is an $n\times n$ matrix of zeros.  \hfill {\small Q.E.F.}   \label{constr:hyperplane}
\end{construction}

\begin{remark}
It is reasonable to construct the sequence $\{\sigma^2_j\}$ taking into account the steps of Weighted Hyperplane Construction.  Noting that $\nu$ sets the characteristic length-scale of measurement space, we define
\begin{align}
\{\sigma^2_j\} = \{|\nu|\times 10^{-j}\}_{j=0}^K \label{eq:sigdef}
\end{align}
where $|\nu|$ is the length of $\nu$.  In practice, we take $K$ to be between 6 and 8, which corresponds to scaling $\sigma^2$ across as many decades relative to $|\nu|$.
\end{remark}

The left plot of Fig.\ \ref{fig:sarsdata} shows the homotopy classifier applied to a set of 2D training data, using the Weighted Hyperplane Classifier as an initial guess, Eq.\ \eqref{eq:sigdef} to set the homotopy parameters, and letting $q=67/115$, which is the true prevalence of the data.  The first three iterations of the homotopy classifier according to Eq.\ \eqref{eq:homotopy} are shown.  As anticipated, the classification boundary improves with decreasing scale parameter $\sigma$, and within three iterations, it has essentially converged.  The right plot of Fig. \ref{fig:sarsdata} shows the result of immediately taking $\sigma\to 0$ without considering a sequence of decreasing scale parameters.  In this case, the optimization becomes trapped in a local minimum and fails to converge to a reasonable classification boundary.  Figure \ref{fig:intuitive} provides an intuitive explanation of why the homotopy method prevents failures of the type illustrated in the right plot of Fig.\ \ref{fig:sarsdata}.  For large values of $\sigma$, the optimization blurs the individual datapoints into a ``cloud'' or density with a characteristic length-scale of $\sigma^2$.  This eliminates local minima associated with individual datapoints being on one side or the other of the classification boundary.  After the first iteration of Eq.\ \eqref{eq:homotopy}, the classification boundary is sufficiently close to the global minimum of unregularized problem that subsequent iterations need only smooth out local minima in the vicinity of the optimal matrix $\A^\star(q)$.  Figure \ref{fig:3d} shows an example of the analysis applied to a 3D dataset, demonstrating the applicability to higher dimensions.  

\begin{figure}
\includegraphics[width=13cm]{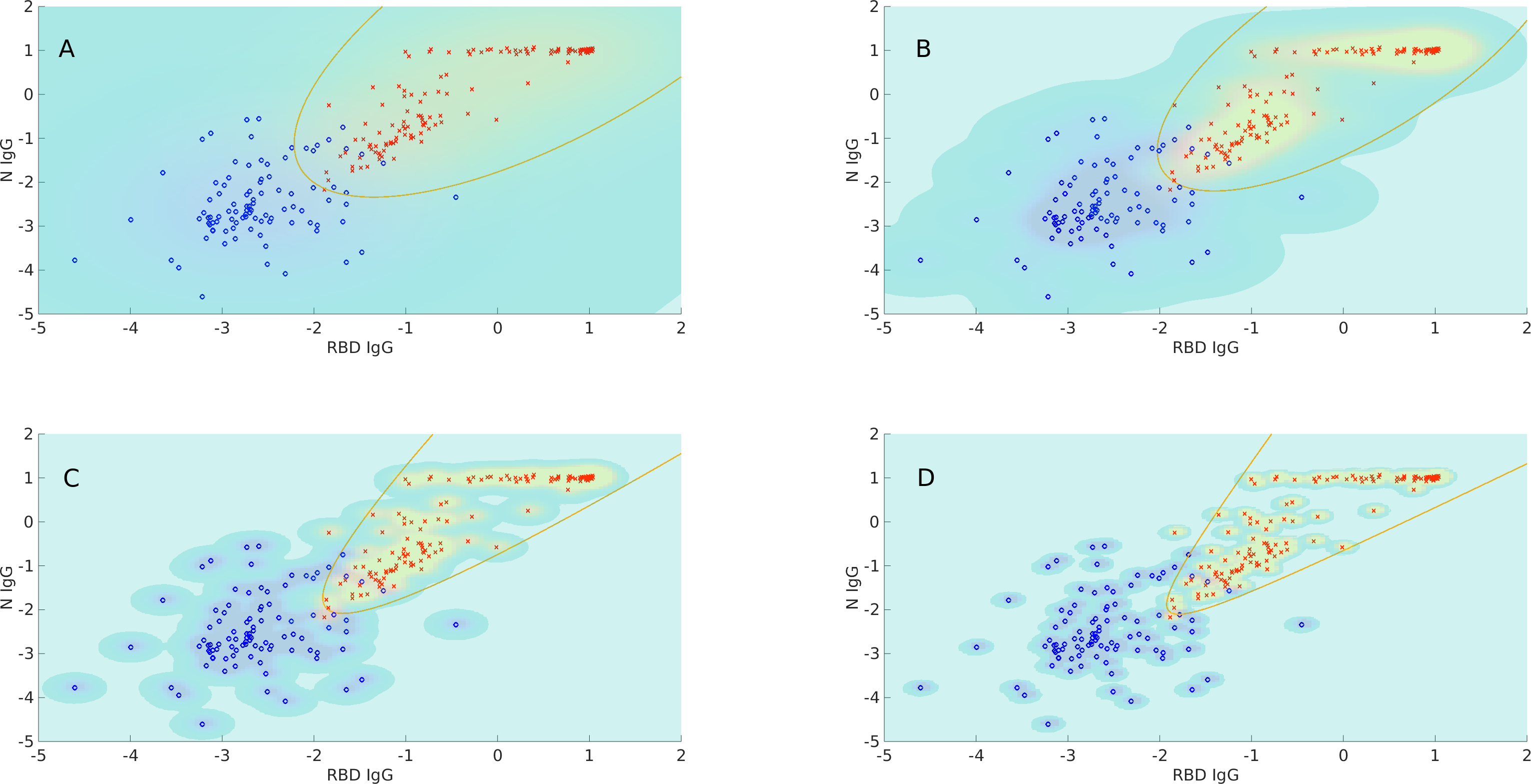}\caption{Illustration of how the homotopy method stabilizes optimization of the classification boundary.  We seek to find a conic section that best separates the positive and negative populations.  However, attempting to work directly with the empirical data is like threading a needle; there is a significant chance that too many data points will end up on the wrong side of the boundary.  See Fig.\ \ref{fig:sarsdata}, for example.  Thus, we blur out the data to temporarily diminish the significance of individual points.  Intuitively, this regularizes  each iteration of the optimization by finding the surface that best separates the blue and yellow shaded areas.  The degree of blurring is proportional to $\sigma^2$.  Going left to right and top to bottom (A to D), the four values of $\sigma^2$ are $\sigma^2_1=1$, $\sigma^2_2=0.25$, $\sigma^2_3=0.1$, and $\sigma^2_4=0.05$.  The final boundary computed for $\sigma^2_j$ is used as the initial guess in the optimization associated with $\sigma^2_{j+1}$.  The color scale is created by convolving the empirical distributions for positive and negative samples with a Gaussian probability density function having a standard deviation $\sigma$.  This is done for illustrative purposes and does not reflect the specifics of how data is blurred in Eq.\ \eqref{eq:homotopy}.  Note that as $\sigma\to 0$, the objective function becomes the empirical classification error, which is the quantity we wish to minimize.}\label{fig:intuitive}
\end{figure}

\begin{figure}\begin{center}
\includegraphics[width=8cm]{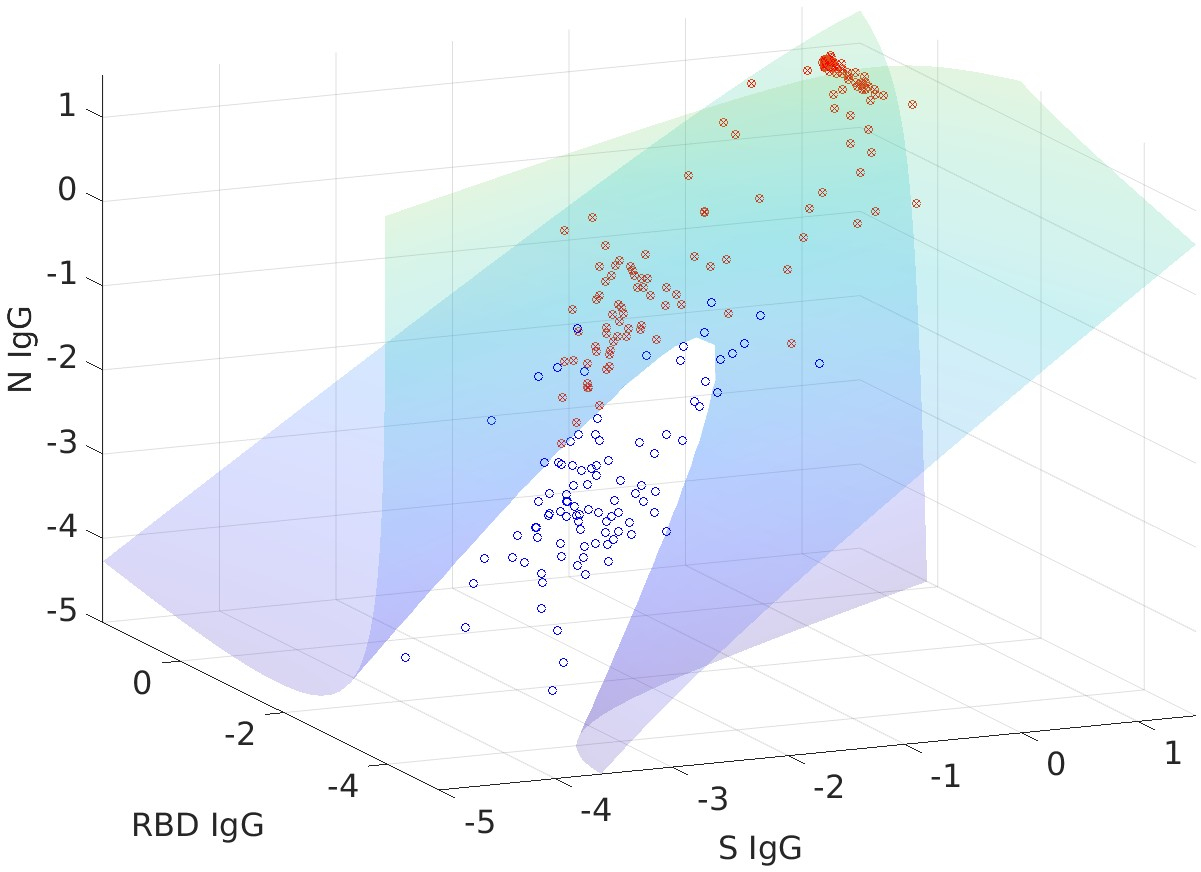}\end{center}\caption{Example of the analysis applied to a 3D dataset.  Due to the perspective, the boundary does not appear to fully separate populations.  However, the accuracy of the classification is roughly 99\%, with an empirical sensitivity and specificity of 98\% and 100\%, respectively.  }\label{fig:3d}
\end{figure}

\section{Uncertainty Level-Sets}
\label{sec:UQ}

Having developed a method for constructing classification boundaries, we now return to the task of uncertainty quantification foreshadowed in Sec.\ \ref{subsec:motivation}.  Our goals in this section are to: (i) show that despite not having access to the conditional PDFs, the function $\mathcal B(\br,\phi^\star(q))$ contains all of the information needed to construct $Z^\star(\br)$; and (ii) develop a practical method for approximating $Z^\star(\br;q)$ given an empirical training population $\Pi$.    To this end, the various interpretations of $q$ play a fundamental role; see Key Remark \ref{rem:prev_interp}.

\subsection{Theoretical Issues}
\label{subsec:uqtheory}

To better understand the relationship between $B^\star(q)$ and $Z^\star(\br;q)$,  recognize that only the ratio of conditional PDFs are needed for classification.  This motivates the following definition.
\begin{definition}
The ratio $R: \Gamma \to [0,\infty]$ defined as
\begin{align}
R(\br) = \frac{N(\br)}{P(\br)} \label{eq:relprob}
\end{align}
is the {\bf relative conditional probability} of a measurement outcome, where we interpret the situation $R(\br)= \infty$ to correspond to $P(\br)=0$.
\end{definition}

\begin{remark}
There is no ambiguity in the definition of $R(\br)$ associated with the case $P(\br)=N(\br)=0$, which yields an indeterminate form for $R(\br)$.  Such points can always be excluded from $\Gamma$, as they have zero measure with respect to $Q(\br)$.  
\end{remark}

In light of the SLS property and SLS postulate, the relative conditional probability is isomorphic to the boundary function $\mathcal B(\br,\phi^\star(q))=0$ in the sense that
\begin{align}
B^\star(q)=\{\br: \mathcal B^\star(\br,q) =0 \} =\{\br: R(\br) = q/(1-q) \} = \{\br: \mathcal B(\br,\phi^\star(q)) =0 \}.  \label{eq:isorelprob} 
\end{align}
Since Eq.\ \eqref{eq:isorelprob} does not alter the structure of the sets $B^\star(q)$, $R(\br)$ must contain all the information necessary for both classification and UQ.  This reduces both tasks to an $(m+1)$ dimensional modeling problem when $\br \in \mathbb R^m$.  To extract $R(\br)$ from a $(m-1)$-dimensional boundary set, observe that  we can ``recover'' an additional degree of freedom from Eq.\ \eqref{eq:isorelprob}, which yields the value of $R(\br)$ on $B$.  We recover the last degree of freedom by allowing $q$ to vary from $0$ to $1$.

To clarify this idea, recognize that the prevalence plays three distinct but related roles in our analysis.  First, $q$ defines a probability measure associated with the fraction of positive individuals in a test population.  Second, $q$ defines a convex combination of error terms in Eqs.\ \eqref{eq:error} and \eqref{eq:empiricalobjective}.  Thus, it could equally be viewed as a parameter of the objective function, not necessarily a property of any given population.  Third, Eq.\ \eqref{eq:isorelprob} shows that $q$ has a one-to-one relationship with $R(\br)$, which means that we can treat $q$ as a function of $R$, or even $\br$.  It therefore stands to reason that $R(\br)$ is not even needed to estimate $Z(\br;U,q)$.  We make this precise as follows.

\begin{definition}
We refer to $\q: \Gamma \to [0,1]$ defined via
\begin{align}
\q(\br) &= \frac{R(\br)}{1+R(\br)} = \frac{N(\br)}{N(\br) + P(\br)} \label{eq:prevfun}
\end{align}
as the {\bf prevalence function}.
\label{def:prevfun}
\end{definition}

Given these definitions, we now prove elementary but fundamental properties of the boundary sets $B^\star(q)$ that establish them as level sets.  
 
\begin{proposition}[Level-Set Representation]
Assume the WLS property holds, and let $q\in (0,1)$.  Then any partition $U^\star(q)$ satisfying the optimal partition convention has the following properties:
\begin{itemize}
\item[I.]  ${B^\star(q) \subset D_p^\star(q')/B^\star(q')}$ and ${B^\star(q') \subset D_n^\star(q)/B^\star(q)}$ for $q' > q$, and in particular, $B^\star(q) \cap B^\star(q') = \emptyset$ for $q\ne q'$;
\item[II.] for a fixed $\br$, the optimally assigned class $\hat C(\br,U^\star(q))$ is a monotone increasing function of $q$;
\item[III.] for every $\br\in \Gamma$, there exists one and only one value of $q$ for which $\br \in B^\star(q)$, and in particular, $\br  \in  B^\star(q) \iff \q(\br)=q$ and ${\mathcal B^\star(\br,\q(\br)) = 0}$;
\item[IV.] $\q(\br)\in (0,1)$ is the value of $q$ for which $\br$ has a 50\% probability of belonging to either class, so that $B^\star(q)$ is the 50\% probability level-set;
\end{itemize}\label{prop:ls}
\end{proposition}
\begin{proof}

Claim I is directly verified from the structure of $D_p^\star(q)$ and $D_n^\star(q)$ and the optimal partition convention, where the latter ensures that  the inequality structure of Eqs.\ \eqref{eq:Dp} and \eqref{eq:Dn} always holds pointwise.  

To prove claim II, fix $\br$ and $q\in (0,1)$.  Because $U^\star(q)$ is a partition, one of three possibilities must hold: $\br\in D_p^\star(q) / B^\star(q)$, $\br \in B^\star(q)$, or $\br \in D_n^\star(q) / B^\star(q)$.  Observe also that $qP(\br)$ is a strictly monotone increasing function of $q$, and $(1-q)N(\br)$ is a strictly monotone decreasing function of $q$.  This implies that $D_p^\star(q) \subset D_p^\star(q')$ if $q < q'$.  Thus in the case that $\br\in D_p^\star(q) / B^\star(q)$, the class assignment of $\br$ remains unchanged as $q$ increases.  If $\br \in B^\star(q)$, then $\hat C(\br;U^\star(q))$ depends on the partition, but by claim I, $\hat C(\br;U^\star(q'))=1$ for any $q'>q$.  Similar arguments imply that $\hat C(\br;U^\star(q'))$ cannot decrease if  $\br \in D_n^\star(q) / B^\star(q)$ and $q'> q$.  

To prove III, note that $R(\br)=N(\br)/P(\br)$ is a mapping $R: \Gamma \to [0,\infty]$.  Moreover, the mapping ${\mathcal R: [0,1] \to [0,\infty]}$ defined as $\mathcal R(\q) = \frac{\q}{1-\q}$ is one-to-one, since $\mathcal R(\q)$ is a strictly monotone increasing function of $\q$.  The boundary set $B^\star(q)$ is defined as the set of $\br$ for which $R(\br) = \mathcal R(q)$, and since the ranges of the two are identical and $\mathcal R(q)$ is invertible, there must exist a $q$ such that $R(\br)=q/(1-q)$.  In fact, this $q$ is $\q(\br)$, which follows from Def.\ \ref{def:prevfun}.  Moreover, since $R(\br)$ is a function, every $\br$ maps to only one value in the range of $R$, and hence one value of $q$.  This proves the main claim of III, and  ${\mathcal B^\star(\br,\q(\br)) = 0}$ follow from the WLS property.  

To prove IV, observe that $Z^\star(\br)$ can be rewritten as
\begin{align}
Z^\star(\br;q) = \frac{q [1-\q(\br)] \mathbb I(\br\in D_p^\star) + (1-q)\q(\br) \mathbb I(\br\in D_n^\star)}{q[1-\q(\br)]+(1-q)\q(\br)}. 
\end{align}
The range of $\q(\br)$ is $[0,1]$.  Thus, when $q\in (0,1)$ and $q=\q(\br)$ for a fixed value of $\br$, one finds $Z^\star(\br)=0.5$, which is the probability correctly classifying the sample.  Note that $Z^\star(\br)$ is indeterminate when $q=\q(\br)\in \{0,1\}$.
\end{proof}

\noindent{\bf Interpretation:} Proposition \ref{prop:ls} defines what we mean by a collection of level sets.  Four properties must hold: level-sets must be disjoint (I), be well-ordered with respect to one another (I) and fixed $\br$ (II), cover $\Gamma$ (III), and be physically meaningful (IV).    

\medskip

We are now in a position to determine the sense in which $\mathcal B(\br;\phi^\star(q))= \mathcal B^\star(\br;q)$.  However, it is possible to pose this as a consequence of a more general result by first taking into account the following definition.

\begin{definition}[Partition Induced by a Classifier]
Let $\hat C(\br;q)$ be an arbitrary classifier, and let $D_n = \{\br: \hat C(\br;q)=0\}$ and $D_p = \{\br: \hat C(\br;q)=1\}$.  We say that $U=\{D_n,D_p\}$ is the \textbf{partition induced by the classifier.}
\end{definition}

\begin{proposition}[Class-Switching Representation]
Let the WLS property and WLS postulate hold.  Let $\hat C(\br;q)$ be an arbitrary classifier inducing a partition $U^\dagger$ that minimizes $\mathcal E(U;q)$ for $q\in(0,1)$, and assume that $U^\dagger$ satisfies the optimal partition convention.  Then for every $\br$, either: $\q(\br)\in \{0,1\}$; or $\q(\br) = q_l(\br)=q_h(\br)$, where
\begin{align}
q_l(\br)=\sup \{q: \hat C(\br;q) = 0\}, && q_h(\br)= \inf \{q: \hat C(\br;q) = 1\}.
\end{align}
\label{prop:bfr}
\end{proposition}

\begin{proof}
Consider first the case for which $P(\br) > 0$, $N(\br)=0$.  Clearly $\br \in D_p^\star(q)$ for any positive $q$.  If $\hat C(\br;q) = 0$, then by Lemma \ref{lemma1}, $qP(\br) \le (1-q)N(\br)$, which is a contradiction.  Thus, $\br$ cannot fall on any boundary set, and we assign $\q(\br)=0$.  Similarly one can assign $\q(\br)=1$ when $P(\br)=0$ and $N(\br) > 0$.  

If both $P(\br) > 0$ and $N(\br) > 0$, then by Proposition \ref{prop:ls} and Def.\ \ref{def:prevfun}, $\q(\br) \in (0,1)$.  Arguing by contradiction, assume that $\q(\br) = q' < q_l(\br)$.  By Proposition \ref{prop:ls} claim III, $\br \in B^\star(q')$.  Let $q''$ satisfy $q' < q'' < q_l(\br)$.  By claim I of Proposition \ref{prop:ls} and the definition of $B^\star(q')$, the optimally assigned class $\hat C(\br;U^\star(q''))=1$, since $q'' > q'$.  But by assumption and Lemma \ref{lemma1}, $\hat C(\br;q'') =0$, since $q'' < q_l(\br)$.  Thus, $\hat C(\br;q)$ is not optimal, which is a contradiction.  Similarly, one finds that $\q(\br) \ngtr q_h(\br)$.  The equality $q_l(\br)=q_h(\br)$ is proved in a similar manner.   That is, assuming the result is false implies there exists a $q''$ such that $q_l(\br)= < q'' < q_h(\br)$.  One finds that either $\hat C$ is not optimal or one of the definitions $q_l(\br)$, $q_h(\br)$, is false, which is a contradiction.  
\end{proof}

\begin{remark}
\Cref{prop:bfr} is general insofar as $\hat C(\br;q)$ need not be a boundary classifier.  Any ML algorithm whose classes induce a partition of some set in $\mathbb R^m$ satisfies the hypotheses, provided the optimal partition convention holds.
\end{remark}

\begin{corollary}
For any $\br$, $Z^\star(\br;q) = Z^\star_l=Z^\star_h$, where 
\begin{align}
Z^\star_k = \frac{q [1-q_k(\br)] \mathbb I(\hat C(\br;q) = 1) + (1-q)q_k(\br) \mathbb I(\hat C(\br;q) = 0)}{q[1-q_k(\br)]+(1-q)q_k(\br)}, && k\in \{l,h\}. \label{eq:rlocalacc}
\end{align}
\end{corollary}

\begin{corollary}
Assume that a boundary classifier $\hat C_{\mathcal B}(\br;q)$ satisfies the SLS postulate.  Then the locus of points solving ${\mathcal B}(\br;q) = 0$ have the property that $\q(\br)=q$.  
\end{corollary}

\begin{figure}
\includegraphics[width=12.5cm]{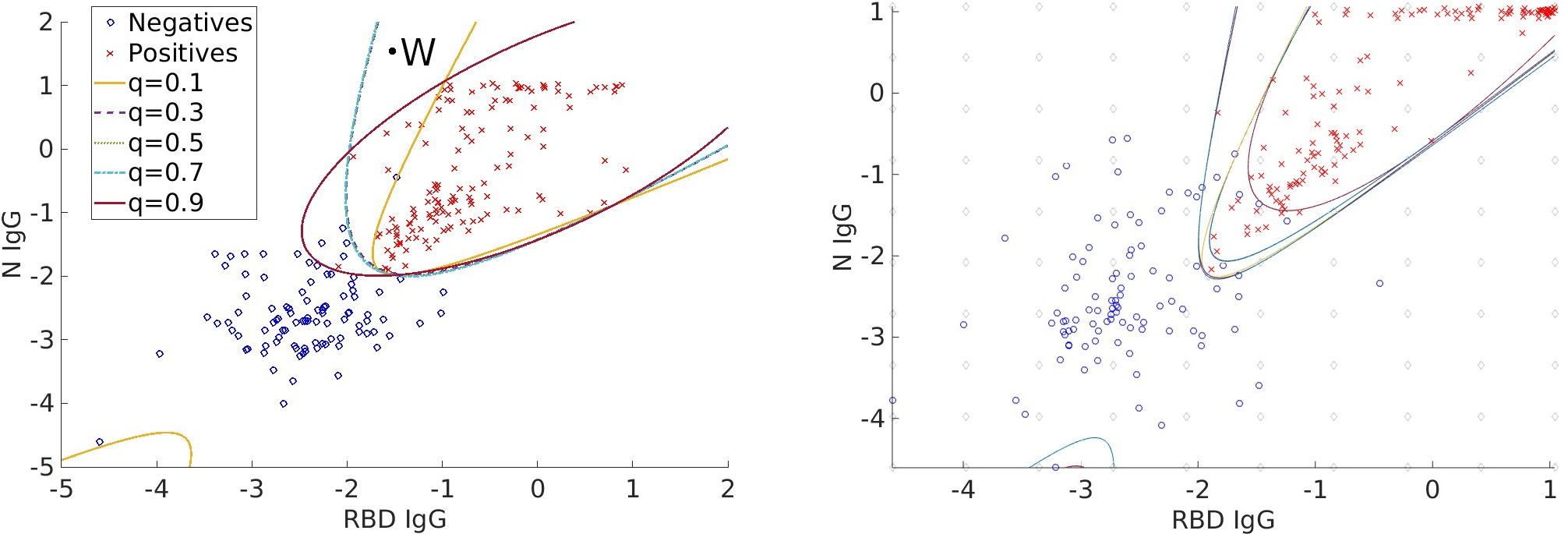}\caption{{\it Left:} Example of optimal classification boundaries that cross.  Each boundary was computed by solving the optimization problem associated with Eq.\ \eqref{eq:objectivesum}, assuming a boundary described by the quadratic model.  Note that the classification boundaries cross.  Thus, the point labeled ``W'' (black dot) has an ambiguously defined classification accuracy.  Specifically, it must simultaneously satisfy the inequalities $\alpha(\br_{\rm W})>0.1$, $\alpha(\br_{\rm W})<0.3$, and $\alpha(\br_{\rm W})>0.9$, which is not possible.  \textit{Right:} Example of classification boundaries computed according to Def.\ \ref{def:lsof}.  We computed classification boundaries for $q=0.05,0.1,...,0.95$ and used 100 uniformly spaced shadow points, which are indicated by faint black diamonds.  We also included a shadow point at the coordinate (-1.91,-2.25) to eliminate overlapping in that region.  Note that $q$ increases going from the upper-right to lower-left.     }\label{fig:crossing}
\end{figure}

\subsection{Practical Implementation}
\label{subsec:uqpractice}

In the context of Eq.\ \eqref{eq:homotopy}, realizing the UQ estimates of Proposition \ref{prop:ls} must address two challenges.  First, it is generally difficult to estimate $\phi^\star(q)$ except on a finite grid of $q$.  This means that we may overestimate the set $[q_l,q_h]$ in Proposition \ref{prop:bfr}, especially when $q_l=q_h$.   Second, the formulation of Sec.\ \ref{sec:classification} does not prohibit crossing between any two estimated level sets with different values of $q$.  While this may not occur in idealized cases, it is problematic given finite training populations.  Figure \ref{fig:crossing} illustrates this problem in the context of the data from Example \ref{ex:sars}.  At points where the level sets cross, we cannot assign meaningful uncertainty estimates via Eq.\ \eqref{eq:rlocalacc}, since $R(\br)$ is no longer single valued.  

In light of claim i.\ of Proposition \ref{prop:ls} and Proposition \ref{prop:bfr}, we interpret crossing level-sets as a violation of the class monotonicity property.  In Fig.\ \ref{fig:crossing}, for example, the class of point ``W'' can decrease with $q$.  This suggests introducing  monotonicity as a constraint of the optimization.  However, doing so is challenging because $\hat C_{\mathcal B}(\br;q)$ integer-valued.  Instead, we apply the constraint directly to $\mathcal B(\br,\phi^\star(q))$ as follows.
\begin{definition}[Level-set Objective Function]
Let $\brho_i$, $i\in\{1,2,...,\tau\}$ be an arbitrary, finite set of points in $\Gamma$.  Let $\Pi$ be a training population.  Let $q_j$ satisfy $0 < q_1 < ... < q_{\theta} < 1$ for some $\theta$, and define 
\begin{align}
\mathcal L_{level}[\phi_1,...,\phi_{\theta};\Pi,\sigma^2] = \sum_{j=1}^{\theta} \mathcal L_{sr}[\phi_j;\Pi,q_j;\sigma^2] 
\end{align}
where the $\phi_j$ are treated as unknown parameters.  Then we call $\mathcal L_{level}$ the {\bf level-set objective function} and subject it to the constraint that
\begin{align}
\mathcal B(\brho_i,\phi(q_j) \le \mathcal B(\brho_i,\phi(q_{j+1})) \label{eq:monconstraint}
\end{align} 
for all pairs $\brho_i,q_j$.  \label{def:lsof}
\end{definition}
\begin{remark}  We refer to the $\brho_i$ as {\bf shadow points}, since they need not be objects of classification {\it per se}.  Their role is to ensure that the level-sets constructed by minimizing of $\mathcal L_{level}$ do not intersect.  Heuristically they do this by requiring $\mathcal B(\brho_i,\phi^\star(q))$ to be monotone increasing function $q$.  In principle, however, this is not sufficient to guarantee that $\hat C_{\mathcal B}(\br;q)$ is also monotone increasing, which can happen if the values of $\phi_j$ are poorly scaled.  In practice we find that scale regularization [e.g.\ via Eq.\ \eqref{eq:regscale}] combined with sufficiently many shadow points overcomes this problem.  See Fig.\ \ref{fig:crossing} for an example of this analysis.  
\end{remark}

\section{Discussion}
\label{sec:discussion}

\subsection{Comparison with Past Works}
\label{subsec:comparison}

Our approach is related to generative and discriminative ML techniques, although it does not neatly fall into either category.   While there is some disagreement within the community as to how best to define these concepts (compare \cite{Disc2,Disc1}), we adopt the perspective of Ref.\ \cite{Disc3}:
\begin{itemize}
\item A generative model quantifies the probability of an outcome $\br$ conditioned on the class.
\item A discriminative model quantifies the probability of belonging to a class, conditioned on outcome $\br$.  
\end{itemize}
These definitions are clearly linked by the concept of conditional probability \cite{Bayes}, but in machine learning they have often been treated as irreconcilable.  

To frame Eq.\ \eqref{eq:homotopy} in the context of generative modeling, consider methods such as linear discriminant and quadratic discriminant analyses (QDA) \cite{Fisher,Hardle,Rao,Venables}.  QDA, for example, assumes that the conditional PDFs are multivariate Gaussians.  The classifier is constructed from the optimal Bayes rule \cite{Bayes} in the spirit of Eqs.\ \eqref{eq:Dp} and \eqref{eq:Dn}.  As a result, the classification boundary is necessarily quadratic, since the ratio of two such PDFs is still an exponential of a quadratic.  However, the converse is not true: a quadratic classification boundary does not imply that $P(\br)$ and $N(\br)$ are Gaussian.  Figure \ref{fig:syndat} is a notable counterexample, since the underlying distribution for the pair $x,y$ arises from the sum of Gaussian and chi-squared random variables.  Thus, the quadratic version of our analysis is more general than QDA.

To place Eqs.\ \eqref{eq:empiricalobjective} and \eqref{eq:homotopy} in the context of discriminative modeling, consider that they resemble support vector machines (SVM) \cite{SVM1,SVM3,SVM2,RW}.  However,  SVMs optimize an empirical objective function that, loosely speaking, attempts to maximize the distance between the classification boundary and the nearest points in the training set (i.e.\ the hard-margin problem), with appropriate generalizations when the data is not linearly separable (the soft-margin problem) \cite{SVM1}.  Importantly, the objective function is not the classification error {\it per se}.  This appears to have limited the ability to connect SVMs to an underlying probabilistic framework, e.g.\ via Bayes optimal methods \cite{SVMProbability2,SVMProbability,SVMNotProb}.  In contrast, the objective function given by Eq.\ \eqref{eq:objectivesum} trivially reduces to the classification error when $\sigma\to 0$, and as a practical matter, we can take $\sigma$ to be sufficiently small so that $\mathcal L$ is an arbitrarily good (albeit empirical) estimate of Eq.\ \eqref{eq:error}.  Moreover, Lemma \ref{lemma1} and Proposition \ref{prop:bfr} imply that the resulting classification boundaries quantify {\it all} relative probabilities of measurement outcomes.  Thus, we interpret these results as directly connecting a discriminative classifier in the spirit of Eq.\ \eqref{eq:homotopy} to its generative counterpart.  

These observations suggest that the distinction between discriminative and generative classifiers requires more nuance.  We propose the following:
\begin{itemize}
\item[(1)] A generative classifier models the relative conditional probability $R(\br)=N(\br)/P(\br)$ directly, which induces the ``discriminative'' boundary given by Eqs.\ \eqref{eq:Dp} and \eqref{eq:Dn} via Lemma \ref{lem:optclass}.
\item[(2)] A discriminative classifier invokes Lemma \ref{lemma1} and Proposition \ref{prop:bfr} to determine (or bound) level-sets of the generative model $R(\br)$.  
\end{itemize} 
In other words, these two tasks are converses of one another in a way that mirrors the relationship  between Lemmas \ref{lem:optclass} and \ref{lemma1}.

These observations have implications for both discriminative and generative modeling.  First, they suggest the need to revisit the objective functions that define the classifiers.  To connect generative and discriminative classifiers, we needed to slightly modify the definition of both.  Second, it may be useful to directly model relative probabilities instead of conditional probabilities, as this is both more general and relevant for classification.  Third,  structure of Eq.\ \eqref{eq:emperror}, which is a convex combination suggests connections to concepts such as linear independence.  This points to possible extensions of our work to impure training data.  Indeed, it would be surprising if SVMs could be extended to unsupervised learning tasks, given the  current understanding of such techniques \cite{SVM3}.  Such topics are the focus of the next manuscript in this series.

\subsection{Limitations and Open Questions}
\label{subsec:limitations}

A key limitation of this work is the need to specify a model associated with the classification boundary.  While the low-order approximation made herein is likely reasonable for many assays, it is not {\it a priori} clear how to estimate the added uncertainties associated with this choice.  We are aware that the choice of model can lead to rare but notable instances wherein data is obviously classified incorrectly, e.g.\ via the second branch of a hyperbola.  While such issues can often be fixed manually after-the-fact, it may not always be clear how to do this in higher-dimensional settings.

We also note several important open directions.  In particular, generalizing the analysis to more than two classes (e.g.\ SARS-CoV-2 infected, vaccinated, and naive) remains challenging, as well as determining the conditions under which impure data can be used for such problems.  A more significant problem arises from waning of antibodies over time and associated variability, which depends on antigen, individual, and exposure level \cite{vaccine}.\footnote{Cross-reactive antibodies that may yield false positives are already accounted for by our analysis, provided the negative population has samples exhibiting cross-reactivity.}  Our analysis does not account for such effects and would require generalizations of Ref.\ \cite{Bedekar22}.   Finally, rigorous convergence estimates with the size of the empirical training sets are important research directions for further grounding the analysis discussed herein.

{\it Acknowledgements:} This work is a contribution of the National Institutes of Standards and Technology and is therefore not subject to copyright in the United States.  RB, CF, and AM were supported under the US National Cancer Institute, Grant U01 CA261276 (The Serological Sciences Network), Massachusetts Consortium on Pathogen Readiness (MassCPR) Evergrande COVID-19 Response Fund Award, and University of Massachusetts Chan Medical School COVID-19 Pandemic Research Fund.  RB was also supported by the National Center for Advancing Translational Sciences, NIH KL2-TR001455 Grant.  Certain commercial equipment, instruments, software, or materials are identified in this paper in order to specify the experimental procedure adequately. Such identification is not intended to imply recommendation or endorsement by the National Institute of Standards and Technology, nor is it intended to imply that the materials or equipment identified are necessarily the best available for the purpose.

{\it Use of all data deriving from human subjects was approved by the NIST and University of Massachusetts Research Protections Offices.}

{\it Data availability:} Data associated with the SARS-CoV-2 assay is available for download as supplemental material to Ref.\ \cite{Raquel1}.  An open-source software package implementing these analyses is under preparation for public distribution.  In the interim, a preliminary version of the software will be made available upon reasonable request.

\bibliographystyle{siamplain}
\bibliography{curved}

\begin{thebibliography}{10}

\bibitem{MLUQ1}
{\sc M.~Abdar, F.~Pourpanah, S.~Hussain, D.~Rezazadegan, L.~Liu,
  M.~Ghavamzadeh, P.~Fieguth, X.~Cao, A.~Khosravi, U.~R. Acharya,
  V.~Makarenkov, and S.~Nahavandi}, {\em A review of uncertainty quantification
  in deep learning: Techniques, applications and challenges}, Information
  Fusion, 76 (2021), pp.~243--297,
  \url{https://doi.org/https://doi.org/10.1016/j.inffus.2021.05.008}.

\bibitem{Homotopy3}
{\sc B.~Addis, M.~Locatelli, and F.~Schoen}, {\em Local optima smoothing for
  global optimization}, Optimization Methods and Software, 20 (2005),
  pp.~417--437, \url{https://doi.org/10.1080/10556780500140029}.

\bibitem{3Sig1}
{\sc A.~Algaissi, M.~A. Alfaleh, S.~Hala, T.~S. Abujamel, S.~S. Alamri, S.~A.
  Almahboub, K.~A. Alluhaybi, H.~I. Hobani, R.~M. Alsulaiman, R.~H. AlHarbi,
  M.-Z. ElAssouli, R.~Y. Alhabbab, A.~A. AlSaieedi, W.~H. Abdulaal, A.~A.
  Al-Somali, F.~S. Alofi, A.~A. Khogeer, A.~A. Alkayyal, A.~B. Mahmoud,
  N.~A.~M. Almontashiri, A.~Pain, and A.~M. Hashem}, {\em Sars-cov-2 s1 and
  n-based serological assays reveal rapid seroconversion and induction of
  specific antibody response in covid-19 patients}, Scientific Reports, 10
  (2020), p.~16561, \url{https://doi.org/10.1038/s41598-020-73491-5}.

\bibitem{Homotopy2}
{\sc E.~L. Allgower and K.~Georg}, {\em Introduction to Numerical Continuation
  Methods}, Society for Industrial and Applied Mathematics, 2003,
  \url{https://doi.org/10.1137/1.9780898719154}.

\bibitem{SVM1}
{\sc M.~Awad and R.~Khanna}, {\em Support Vector Machines for Classification},
  Apress, Berkeley, CA, 2015, pp.~39--66.

\bibitem{Bedekar22}
{\sc P.~Bedekar, A.~J. Kearsley, and P.~N. Patrone}, {\em Prevalence estimation
  and optimal classification methods to account for time dependence in antibody
  levels}, Journal of Theoretical Biology, 559 (2023), p.~111375,
  \url{https://doi.org/https://doi.org/10.1016/j.jtbi.2022.111375}.

\bibitem{SVM3}
{\sc K.~P. Bennett and C.~Campbell}, {\em Support vector machines: Hype or
  hallelujah?}, SIGKDD Explor. Newsl., 2 (2000), p.~1–13,
  \url{https://doi.org/10.1145/380995.380999}.

\bibitem{Raquel1}
{\sc R.~A. Binder, G.~F. Fujimori, C.~S. Forconi, G.~W. Reed, L.~S. Silva,
  P.~S. Lakshmi, A.~Higgins, L.~Cincotta, P.~Dutta, M.-C. Salive, V.~Mangolds,
  O.~Anya, J.~M. Calvo~Calle, T.~Nixon, Q.~Tang, M.~Wessolossky, Y.~Wang, D.~A.
  Ritacco, C.~S. Bly, S.~Fischinger, C.~Atyeo, P.~O. Oluoch, B.~Odwar, J.~A.
  Bailey, A.~Maldonado-Contreras, J.~P. Haran, A.~G. Schmidt, L.~Cavacini,
  G.~Alter, and A.~M. Moormann}, {\em {SARS-CoV-2 Serosurveys: How Antigen,
  Isotype and Threshold Choices Affect the Outcome}}, The Journal of Infectious
  Diseases, 227 (2022), pp.~371--380.

\bibitem{posttest1}
{\sc Z.~Brooks, S.~Das, and T.~Pliura}, {\em Clinicians' probability calculator
  to convert pre-test to post-test probability of sars-cov-2 infection based on
  method validation from each laboratory}, Journal of the International
  Federation of Clinical Chemistry / IFCC, 32 (2021), p.~265.

\bibitem{OldPrevOpt}
{\sc C.~Brownie and J.-P. Habicht}, {\em Selecting a screening cut-off point or
  diagnostic criterion for comparing prevalences of disease}, Biometrics, 40
  (1984), pp.~675--684, \url{http://www.jstor.org/stable/2530910} (accessed
  2023-03-13).

\bibitem{montecarlo}
{\sc R.~E. Caflisch}, {\em Monte carlo and quasi-monte carlo methods}, Acta
  Numerica, 7 (1998), p.~1–49,
  \url{https://doi.org/10.1017/S0962492900002804}.

\bibitem{Cantelli}
{\sc F.~Cantelli}, {\em Sulla probabilista come limita della frequencza}, Rend.
  Accad. Lincei, 26 (1917), p.~39.

\bibitem{SVM2}
{\sc N.~Cristianini and J.~Shawe-Taylor}, {\em An Introduction to Support
  Vector Machines and Other Kernel-based Learning Methods}, Cambridge
  University Press, 2000, \url{https://doi.org/10.1017/CBO9780511801389}.

\bibitem{Bayes}
{\sc L.~Devroye, L.~Gy{\"o}rfi, and G.~Lugosi}, {\em The Maximum Likelihood
  Principle}, Springer New York, New York, NY, 1996, pp.~249--262.

\bibitem{MLE}
{\sc L.~Devroye, L.~Gy{\"o}rfi, and G.~Lugosi}, {\em The Maximum Likelihood
  Principle}, Springer New York, New York, NY, 1996, pp.~249--262.

\bibitem{BoisvertUQ}
{\sc A.~Dienstfrey and R.~Boisvert}, {\em Uncertainty Quantification in
  Scientific Computing : 10th IFIP WG2.5Working Conference, WoCoUQ 2011,
  Boulder, CO, USA, August 1-4, 2011}, 01 2012.

\bibitem{Doob}
{\sc J.~L. Doob}, {\em {The Limiting Distributions of Certain Statistics}}, The
  Annals of Mathematical Statistics, 6 (1935), pp.~160 -- 169,
  \url{https://doi.org/10.1214/aoms/1177732594}.

\bibitem{Homotopy4}
{\sc D.~M. Dunlavy and D.~P. O'Leary}, {\em Homotopy optimization methods for
  global optimization.},  (2005), \url{https://doi.org/10.2172/876373}.

\bibitem{Borel}
{\sc M.~{\'E}mile~Borel}, {\em Les probabilit{\'e}s d{\'e}nombrables et leurs
  applications arithm{\'e}tiques}, Rendiconti del Circolo Matematico di Palermo
  (1884-1940), 27 (1909), pp.~247--271.

\bibitem{Evans}
{\sc L.~Evans}, {\em An Introduction to Stochastic Differential Equations},
  American Mathematical Society, 2012.

\bibitem{EUA}
{\sc FDA}, {\em Eua authorized serology test performance}.
\newblock
  {https://www.fda.gov/medical-devices/coronavirus-disease-2019-covid-19-emergency-use-authorizations-medical-devices/eua-authorized-serology-test-performance},
  2020.
\newblock Accessed: 2020-09-16.

\bibitem{Fisher}
{\sc R.~A. FISHER}, {\em The use of multiple measurements in taxonomic
  problems}, Annals of Eugenics, 7 (1936), pp.~179--188.

\bibitem{ROC}
{\sc C.~M. Florkowski}, {\em Sensitivity, specificity, receiver-operating
  characteristic (roc) curves and likelihood ratios: communicating the
  performance of diagnostic tests}, The Clinical biochemist. Reviews, 29 Suppl
  1 (2008), pp.~S83--S87.

\bibitem{SVMProbability2}
{\sc V.~Franc, A.~Zien, and B.~Schölkopf}, {\em Support vector machines as
  probabilistic models.}, 01 2011, pp.~665--672.

\bibitem{3Sig2}
{\sc L.~Grzelak, S.~Temmam, C.~Planchais, C.~Demeret, L.~Tondeur, C.~Huon,
  F.~Guivel-Benhassine, I.~Staropoli, M.~Chazal, J.~Dufloo, D.~Planas,
  J.~Buchrieser, M.~M. Rajah, R.~Robinot, F.~Porrot, M.~Albert, K.-Y. Chen,
  B.~Crescenzo-Chaigne, F.~Donati, F.~Anna, P.~Souque, M.~Gransagne,
  J.~Bellalou, M.~Nowakowski, M.~Backovic, L.~Bouadma, L.~L. Fevre, Q.~L.
  Hingrat, D.~Descamps, A.~Pourbaix, C.~Laouénan, J.~Ghosn, Y.~Yazdanpanah,
  C.~Besombes, N.~Jolly, S.~Pellerin-Fernandes, O.~Cheny, M.-N. Ungeheuer,
  G.~Mellon, P.~Morel, S.~Rolland, F.~A. Rey, S.~Behillil, V.~Enouf,
  A.~Lemaitre, M.-A. Créach, S.~Petres, N.~Escriou, P.~Charneau, A.~Fontanet,
  B.~Hoen, T.~Bruel, M.~Eloit, H.~Mouquet, O.~Schwartz, and S.~van~der Werf},
  {\em A comparison of four serological assays for detecting
  anti\&\#x2013;sars-cov-2 antibodies in human serum samples from different
  populations}, Science Translational Medicine, 12 (2020), p.~eabc3103.

\bibitem{posttest4}
{\sc W.~Gu and M.~S. Pepe}, {\em {Estimating the diagnostic likelihood ratio of
  a continuous marker}}, Biostatistics, 12 (2010), pp.~87--101,
  \url{https://doi.org/10.1093/biostatistics/kxq045}.

\bibitem{3Sig3}
{\sc A.~Hachim, N.~Kavian, C.~A. Cohen, A.~W.~H. Chin, D.~K.~W. Chu, C.~K.~P.
  Mok, O.~T.~Y. Tsang, Y.~C. Yeung, R.~A. P.~M. Perera, L.~L.~M. Poon, J.~S.~M.
  Peiris, and S.~A. Valkenburg}, {\em Orf8 and orf3b antibodies are accurate
  serological markers of early and late sars-cov-2 infection}, Nature
  Immunology, 21 (2020), pp.~1293--1301,
  \url{https://doi.org/10.1038/s41590-020-0773-7}.

\bibitem{vaccine}
{\sc K.~Hajissa, A.~Mussa, M.~I. Karobari, M.~A. Abbas, I.~K. Ibrahim, A.~A.
  Assiry, A.~Iqbal, S.~Alhumaid, A.~A. Mutair, A.~A. Rabaan, P.~Messina, and
  G.~A. Scardina}, {\em The sars-cov-2 antibodies, their diagnostic utility,
  and their potential for vaccine development}, Vaccines, 10 (2022),
  \url{https://doi.org/10.3390/vaccines10081346}.

\bibitem{SMC2}
{\sc P.~Hall}, {\em On the rate of convergence of orthogonal series density
  estimators}, Journal of the Royal Statistical Society. Series B
  (Methodological), 48 (1986), pp.~115--122.

\bibitem{Hardle}
{\sc W.~K. H{\"a}rdle and L.~Simar}, {\em Discriminant Analysis}, Springer
  Berlin Heidelberg, Berlin, Heidelberg, 2015, pp.~407--424.

\bibitem{posttest3}
{\sc S.~R. Hayden and M.~D. Brown}, {\em Likelihood ratio: A powerful tool for
  incorporating the results of a diagnostic test into clinical decisionmaking},
  Annals of Emergency Medicine, 33 (1999), pp.~575--580.

\bibitem{Hoeffding}
{\sc W.~Hoeffding}, {\em Probability inequalities for sums of bounded random
  variables}, Journal of the American Statistical Association, 58 (1963),
  pp.~13--30.

\bibitem{Disc2}
{\sc T.~Jebara}, {\em Generative Versus Discriminative Learning}, Springer US,
  Boston, MA, 2004, pp.~17--60.

\bibitem{Luke23_1}
{\sc R.~A. Luke, A.~J. Kearsley, N.~Pisanic, Y.~C. Manabe, D.~L. Thomas, C.~D.
  Heaney, and P.~N. Patrone}, {\em Modeling in higher dimensions to improve
  diagnostic testing accuracy: Theory and examples for multiplex saliva-based
  sars-cov-2 antibody assays}, PLOS ONE, 18 (2023), pp.~1--11,
  \url{https://doi.org/10.1371/journal.pone.0280823}.

\bibitem{hypermutation}
{\sc A.~Martin, R.~Chahwan, J.~Y. Parsa, and M.~D. Scharff}, {\em Chapter 20 -
  somatic hypermutation: The molecular mechanisms underlying the production of
  effective high-affinity antibodies}, in Molecular Biology of B Cells (Second
  Edition), F.~W. Alt, T.~Honjo, A.~Radbruch, and M.~Reth, eds., Academic
  Press, London, second edition~ed., 2015, pp.~363--388.

\bibitem{Disc3}
{\sc T.~Mitchell}, {\em Machine Learning}, McGraw Hill series in computer
  science, McGraw Hill, 2017.

\bibitem{genantibody}
{\sc P.~R. Murray}, {\em 16 - the clinician and the microbiology laboratory},
  in Mandell, Douglas, and Bennett's Principles and Practice of Infectious
  Diseases (Eighth Edition), J.~E. Bennett, R.~Dolin, and M.~J. Blaser, eds.,
  W.B. Saunders, Philadelphia, eighth edition~ed., 2015, pp.~191--223.

\bibitem{Disc1}
{\sc A.~Ng and M.~Jordan}, {\em On discriminative vs. generative classifiers: A
  comparison of logistic regression and naive bayes}, Advances in neural
  information processing systems, 14 (2001).

\bibitem{Nocedal}
{\sc J.~Nocedal and S.~Wright}, {\em Numerical Optimization}, Springer Series
  in Operations Research and Financial Engineering, Springer New York, 2006.

\bibitem{posttest2}
{\sc R.~Parikh, S.~Parikh, E.~Arun, and R.~Thomas}, {\em Likelihood ratios:
  Clinical application in day-to-day practice}, Indian Journal of
  Ophthalmology, 57 (2009).

\bibitem{Patrone22_2}
{\sc P.~Patrone and A.~Kearsley}, {\em Minimizing uncertainty in prevalence
  estimates}, 2022, \url{https://arxiv.org/abs/2203.12792}.

\bibitem{Patrone22_1}
{\sc P.~N. Patrone, P.~Bedekar, N.~Pisanic, Y.~C. Manabe, D.~L. Thomas, C.~D.
  Heaney, and A.~J. Kearsley}, {\em Optimal decision theory for diagnostic
  testing: Minimizing indeterminate classes with applications to saliva-based
  sars-cov-2 antibody assays}, Mathematical Biosciences, 351 (2022), p.~108858,
  \url{https://doi.org/https://doi.org/10.1016/j.mbs.2022.108858}.

\bibitem{Patrone21_1}
{\sc P.~N. Patrone and A.~J. Kearsley}, {\em {Classification under uncertainty:
  data analysis for diagnostic antibody testing}}, Mathematical Medicine and
  Biology: A Journal of the IMA, 38 (2021), pp.~396--416,
  \url{https://doi.org/10.1093/imammb/dqab007}.

\bibitem{Patrone17}
{\sc P.~N. Patrone and T.~W. Rosch}, {\em {Beyond histograms: Efficiently
  estimating radial distribution functions via spectral Monte Carlo}}, The
  Journal of Chemical Physics, 146 (2017),
  \url{https://doi.org/10.1063/1.4977516}.
\newblock 094107.

\bibitem{SVMProbability}
{\sc J.~Platt}, {\em Probabilistic outputs for support vector machines and
  comparisons to regularized likelihood methods}, Adv. Large Margin Classif.,
  10 (2000).

\bibitem{Homotopy5}
{\sc Z.~Qiu, L.~Peng, A.~Manatunga, and Y.~Guo}, {\em A smooth nonparametric
  approach to determining cut-points of a continuous scale}, Computational
  Statistics and Data Analysis, 134 (2019), pp.~186--210.

\bibitem{Rao}
{\sc C.~R. Rao}, {\em The utilization of multiple measurements in problems of
  biological classification}, Journal of the Royal Statistical Society: Series
  B (Methodological), 10 (1948), pp.~159--193.

\bibitem{RW}
{\sc C.~Rasmussen and C.~Williams}, {\em Gaussian Processes for Machine
  Learning}, Adaptative computation and machine learning series, University
  Press Group Limited, 2006.

\bibitem{SmithUQ}
{\sc R.~Smith}, {\em Uncertainty Quantification: Theory, Implementation, and
  Applications}, Computational Science and Engineering, Society for Industrial
  and Applied Mathematics, 2013.

\bibitem{SVMNotProb}
{\sc I.~Steinwart and A.~Christmann}, {\em Support vector machines}, in
  Information Science and Statistics, 2008.

\bibitem{NA}
{\sc J.~Stoer, R.~Bartels, W.~Gautschi, R.~Bulirsch, and C.~Witzgall}, {\em
  Introduction to Numerical Analysis}, Texts in Applied Mathematics, Springer
  New York, 2002.

\bibitem{pspace}
{\sc D.~Stroock}, {\em Probability Theory: An Analytic View}, Cambridge
  University Press, 2010.

\bibitem{Tao}
{\sc T.~Tao}, {\em An Introduction to Measure Theory}, Graduate studies in
  mathematics, American Mathematical Society, 2013.

\bibitem{Venables}
{\sc W.~N. Venables and B.~D. Ripley}, {\em Classification}, Springer New York,
  New York, NY, 2002, pp.~331--351.

\bibitem{SMC3}
{\sc G.~G. Walter}, {\em Properties of hermite series estimation of probability
  density}, The Annals of Statistics, 5 (1977), pp.~1258--1264.

\bibitem{SMC1}
{\sc G.~S. Watson}, {\em Density estimation by orthogonal series}, The Annals
  of Mathematical Statistics, 40 (1969), pp.~1496--1498.

\bibitem{Homotopy1}
{\sc L.~T. Watson and R.~T. Haftka}, {\em Modern homotopy methods in
  optimization}, Computer Methods in Applied Mechanics and Engineering, 74
  (1989), pp.~289--305.

\bibitem{totprob}
{\sc D.~Zwillinger and S.~Kokoska}, {\em CRC Standard Probability and
  Statistics Tables and Formulae}, CRC Press, 1999.

\end{thebibliography}

\end{document}